\documentclass[runningheads]{llncs}
\usepackage[T1]{fontenc}
\usepackage{graphicx}
\usepackage{lmodern} 
\usepackage{caption}
\DeclareCaptionFont{llncs_caption}{\fontsize{9}{10}\selectfont} 
\usepackage{amsmath}
\usepackage{amsfonts}
\usepackage{dsfont}
\usepackage{graphicx}
\usepackage{subcaption}
\usepackage{mwe} 
\usepackage{color}
\usepackage{booktabs}
\usepackage{cleveref}
\AtBeginDocument{%
  \crefname{equation}{Eq.}{Eqs.}%
  \Crefname{equation}{Equation}{Equations}%
  \crefname{figure}{Fig.}{Figs.}%
  \Crefname{figure}{Figure}{Figures}%
  \crefname{table}{Table}{Tables}%
  \Crefname{table}{Table}{Tables}%
}

\newcommand\scomment[1]{}

\DeclareMathOperator{\PP}{\mathrm{P}}

\DeclareMathOperator{\R}{\mathbb{R}}

\DeclareMathOperator{\rank}{rank}

\DeclareMathOperator{\diag}{diag}

\DeclareMathOperator{\Unif}{Unif}

\newcommand\ones[0]{\mathbf 1}

\newcommand\norm[1]{\left\lVert#1\right\rVert}

\newcommand\angleb[1]{\left\langle#1\right\rangle}

\newcommand\tsc[1]{\text{\textsc{#1}}}
\newcommand\ttt[1]{\texttt{#1}}
\newcommand\mc[1]{\mathcal{#1}}
\newcommand\mb[1]{\mathbb{#1}}

\newcommand\mt[1]{\mathit{#1}}

\usepackage[noend]{algpseudocode}
\usepackage{algorithm}
\usepackage{multirow}
\usepackage{subcaption}

\newcommand\HH{\mc H}
\newcommand\HHhat{\hat{\mc H}}

\newcommand\forw{\mathbf{Forw}}
\newcommand\back{\mathbf{Back}}

\newcommand\azizz{Azizzadenesheli et al. }
\newcommand\Obs{\textup{\text{Obs}}}

\newcommand\hist{\mathit{hist}}
\newcommand\ttest{\mathit{test}}
\newcommand\hists{\mathit{hists}}

\newcommand\nobs{n^{\mt{obs}}}

\newcommand\full{\mt{full}}
\newcommand\Afull{\mc A_{\mt{full}}}

\newcommand\acts{\text{acts}}

\newcommand\simobs{M^{ao} {M^a}^{-1}}

\newcommand{\aonnobs}{(|\mc A||\mc O|)^{2(\nobs+1)}}

\newcommand{\aopnobs}{(|\mc A||\mc O|)^{\nobs+2}}
\newcommand{\aominnobs}{(|\mc A||\mc O|)^{2\nobs}}
\newcommand{\aonnobsmin}{(|\mc A||\mc O|)^{2\nobs+1}}

\newcommand\eTs{\mc{\hat T}}
\newcommand\eZs{\mc{\hat Z}}
\newcommand\eTmat[1]{\hat T^{#1}}
\newcommand\eOmat[1]{\hat O^{#1}}

\newif\ifwithappendix
\withappendixtrue
\ifwithappendix\else\fi

\begin{document}
\title{Toward Learning POMDPs Beyond \\ Full-Rank Actions and State Observability}
\titlerunning{Learning POMDPs Beyond Full-Rank Actions and State Observability}
%
\author{Seiji Shaw\inst{1} \and
Travis Manderson\inst{1} \and
Chad Kessens\inst{2} \and
Nicholas Roy \inst{1}}
\authorrunning{S. Shaw et al.}
%
\institute{MIT Computer Science and Artificial Intelligence Lab, 
\\51 Vassar St, Cambridge, MA 02139, USA 
\\\email{\{seijis, travislm\}@mit.edu, nickroy@csail.mit.edu}
\and
DEVCOM Army Research Laboratory\\
2800 Powder Mill Road, Adelphi, MD 20783, USA\\
\email{chad.c.kessens.civ@army.mil}}
\maketitle              
\begin{abstract}
We are interested in enabling autonomous agents to learn and reason about systems with hidden states, such as locking mechanisms.
We cast this problem as learning the parameters of a discrete Partially Observable Markov Decision Process (POMDP).
The agent begins with knowledge of the POMDP's actions and observation spaces, but not its state space, transitions, or observation models. 
These properties must be constructed from a sequence of actions and observations. 
Spectral approaches to learning models of partially observable domains, such as Predictive State Representations (PSRs), learn representations of state that are sufficient to predict future outcomes. 
PSR models, however, do not have explicit transition and observation system models that can be used with different reward functions to solve different planning problems.
Under a mild set of rankness assumptions on the products of transition and observation matrices, we show how PSRs learn POMDP matrices up to a similarity transform, and this transform may be estimated via tensor decomposition methods.
Our method learns observation matrices and transition matrices up to a partition of states, where the states in a single partition have the same observation distributions corresponding to actions whose transition matrices are full-rank.
Our numerical experiments suggest that explicit observation and transition likelihoods can be leveraged to generate new plans for different goals and reward functions after the model has been learned.
We also show that learning a POMDP beyond a partition of states is impossible from sequential data by constructing two POMDPs that agree on all observation distributions but differ in their transition dynamics.
\end{abstract}

\vspace{-.2in}
\section{Introduction}\label{sec:introduction}
\vspace{-.1in}





\scomment{
When planning and acting in the real world, humans learn and reason about discrete hidden information.

Use doors, toolboxes.

Inspired by the application of learning discrete locking structures... cite Martin-Martin, and Ben's sensorized box (if properly sensorized).
}

When planning and acting in the real world, intelligent agents must learn and reason about hidden information.
Of great inspiration to us is the work of Baum et al. \cite{baumOpeningLockbox2017}, which shows that a real autonomous robot can infer a cabinet's locking mechanism from a hypothesis set of mechanisms through interaction.
We are interested in a symbolic variant of the problem where autonomous agents must learn, through interaction, the dynamics of a system with hidden states, without any knowledge of the system state and transitions beforehand.
The agent should also compute explicit estimates of transition and observation likelihoods to support downstream operations that manipulate the model, such as task specification to direct agent behavior.
Our problem is modeled as learning the parameters of a discrete Partially Observable Markov Decision Process (POMDP) from a sequence of actions and observations acquired through random exploration.

One common approach to learning a representation of a probabilistic latent-variable model like a POMDP is to apply a spectral decomposition to a matrix that encodes correlations of the observable random variables \cite{hsuSpectralAlgorithm2012,balleSpectralLearning2014}.
For POMDPs, spectral methods may be applied to a \textit{Hankel matrix}, which represents the correlation between past and future observations conditioned on a sequence of past and future actions.
The decomposition of this matrix can be used to derive a (linear) Predictive State Representation \cite{bootsClosingLearningplanning2011,balleSpectralLearning2014}.
The `state' of the PSR is a sufficient statistic that can be used to predict the likelihood of future observations given a possible sequence of actions.
This prediction capability allows PSRs to be used as black-box models for reinforcement learning \cite{liuWhenPartially2022,zhanPACReinforcement2022}; however, transition and observation likelihoods cannot be directly recovered from a PSR. 
This lack of interpretability makes these models difficult to manipulate, like changing the goal or reward function for planning. 
If a goal state of the agent changes, then the PSR must be relearned for the new task, since the underlying state cannot be accessed.

There are other POMDP-learning algorithms that yield estimates of the full model, e.g. observation and transition likelihoods, but under assumptions that ultimately restrict the class of POMDPs that can be learned.
Approaches introduced by \azizz \cite{azizzadenesheliReinforcementLearning2016} and Guo et al. \cite{guoPACRL2016} utilize tensor decompositions to recover observation distributions for each action whose transition matrix is full-rank.
To recover the transitions, however, these approaches must also make the assumption that for each action, the corresponding diagonal observation matrices must be unique for every state, which implies every state has a unique observation distribution.
While, full-rank transitions are common when modeling many real-world POMDPs, especially when actions may `fail' with some probability,  
many real-world systems have \textit{aliased states}, e.g. states that do not have distinct observation distributions associated with every action. 
Systems that fall in this class include the locking mechanisms of Baum et al. \cite{baumOpeningLockbox2017} or many standard POMDPs in the literature, like Tiger \cite{kaelbling_planning_1998}.

We investigate the relationship between PSRs and tensor decomposition methods to learn a broader class of POMDPs than existing tensor methods. 
A result established by Carlyle and Paz \cite{carlyleRealizationsStochastic1971a} states that PSRs learn transitions and diagonal observation matrices up to an unknown basis. 
We then reformulate tensor decomposition methods to estimate the unknown basis to recover the original basis. 
Our modification of tensor decomposition methods for hidden state inference allows us to simultaneously leverage all observation distributions from \textit{all} actions with full-rank transition methods all at once, rather than a per-action basis like previous approaches \cite{azizzadenesheliReinforcementLearning2016,guoPACRL2016}.
Should the collection of observation distributions of all full-rank actions be unique for each state, like  Tiger, we may recover the full POMDP.
Should there exist states that share the same set of observation distributions when aggregated across actions, we learn transitions between partitions of states, where states in a single partition share the same observation distributions over all actions.
We also show that when restricted to sequential data, learning transition and observation up to observability partitions \textit{cannot be improved}. 
We construct an example of two POMDPs whose dynamics differ between aliased states but yield the same distribution over all future observations under an arbitrary sequence of actions.

Learning explicit transition and observation matrices is valuable because these models enable reasoning over environment dynamics. 
Whereas black-box PSRs only provide predictive likelihoods of observation sequences, access to explicit transition matrices and diagonal observation matrices allows for the specification of rewards after the model has been learned.
Our experimental results suggest that our method can correctly learn partition-level transitions and observations and that these likelihoods are necessary to correctly direct agent behavior in POMDPs with very noisy observations.

\vspace{-.1in}
\section{Problem Setting}\label{sec:problem}
\vspace{-.1in}


We assume that the ground truth system can be described as a discrete POMDP, a tuple $(\mc S, \mc T, \mc A, \mc O, \mc Z, b_0, R, \gamma)$.
The set $\mc S = \{s^1, s^2, \dots\}$ is a discrete set of states, $\mc A$ is a discrete set of actions, and $\mc O = \{o^1, o^2, \dots\}$ is a discrete set of observations.
$\mc T = \{T^a : a \in \mc A\}$ denotes a set of row-stochastic state transition matrices.
The element $T^a_{ij} = \PP(s_{t+1}=s^j | s_t=s^i, a_t=a)$ denotes the probability of transition to state $s^j$ from state $s^i$ after taking action $a$ at time $t$.
The set $\mc Z = \{O^{ao}: (a, o) \in \mc A \times \mc O\}$ describes a collection of diagonal matrices, where $O^{ao}_{ii} = \PP(o_t=o| s_t=s^i, a_t=a)$ denotes the emission probability of $o$ under action $a$ when \textit{leaving} state $s^i$.
The distribution $b_0 \in \Delta(\mc S)$ describes the distribution over the initial state.
The constant $\gamma \in (0, 1)$ is the reward discount factor.

The agent begins acting in a POMDP with access to the action and observation spaces $\mc A$ and $\mc O$.
Under a uniform, memoryless random exploration policy $a_t \sim \Unif(\mc A)$ for all $t \geq 1$, the agent collects a dataset $\mc D$, which is a long string of actions and observations $\mc D = (a_1, o_1, a_2, o_2, \dots)$.
From this data, we wish the agent to estimate the number of hidden states $|\mc S|$, transition matrices $\eTs = \{\eTmat{a} : a \in \mc A\}$, and diagonal observation matrices  $\eZs = \{\eOmat{ao} : (a, o) \in \mc A \times \mc O\}$. 
We may also require the agent to learn a tabular reward $R$ function by including rewards as observations \cite{izadiPointBasedPlanning2008}.
We evaluate the approach by measuring the error of the estimated model parameters against those of the  ground-truth POMDP.
Another evaluation gauges the performance of the agent under a planning algorithm after the POMDP is inferred from $\mc D$.
The last is by evaluating the behavior of a planner at user-designated task after model learning.


For notational convenience, we will often need to refer to strings of actions and observations. 
We call a string $(a_1, o_1, \dots, a_t, o_t)$ that is observed by the agent in the past a \textit{history}, often abbreviated as $\hist$.
A string $(a_{t+1}, o_{t+1}, \dots, a_n, o_n)$ the agent may observe in the future will be called a $\ttest$ \cite{littmanPredictiveRepresentations2001,singhPredictiveState2004}. 
To abbreviate the expression of the likelihood of observations conditioned on the actions, we write \begin{align*}
    p(hist, test) &= \PP(o_1, \dots, o_t, o_{t+1}, \dots, o_n | a_1, \dots, a_t, a_{t+1}, \dots a_n),\\  
    p(hist, s_t = s^i) &= \PP(o_1, \dots, o_t, s_t = s^i | a_1, \dots, a_t), \\
    p(test | s_t = s^i) &=  p(o_{t+1}, \dots, o_n | a_{t+1}, \dots, a_n, s_{t+1}=s^i)\text{ \dots and so on.}
\end{align*}
Occasionally, it will be convenient to stack the diagonals of matrices in $\mc Z$ associated with the same action $a$ into $|\mc S| \times |\mc O|$ \textit{row-stochastic matrices} $\Obs^a$, where $\Obs_{i,j}^a = \PP(o_t=o^j | s_t = s^i, a_t=a)$.
To distinguish between the two, we refer to the collection $\mc Z$ as \textit{diagonal observation matrices} and the latter matrices $\{\Obs^a: a \in \mc A\}$ as \textit{row-stochastic observation} matrices.


\vspace{-.1in}
\section{Learning Predictive State Representations}\label{sec:example}
\vspace{-.1in}
\subsection{Forward, Backward, and Hankel Matrices}\label{sec:hankel-construction}

To estimate systems of hidden state, a natural place to start is to form an array that expresses the correlation between the observable random variables.
A \textit{Hankel matrix} is an instance of these arrays that encodes the joint likelihoods of past and future action-observation trajectories. 
In this section, we derive the Hankel matrix given knowledge of the ground truth POMDP.
Our construction starts with two intermediate factors, called the \textit{forward} and \textit{backward} matrix, which we will multiply together to form the Hankel matrix.\footnote{This factorized construction is inspired by the construction of a related matrix, called the \textit{System Dynamics Matrix}, by Singh et al. \cite{singhPredictiveState2004}. The entries of the System Dynamics Matrix (SDM) are the likelihood of a given test \textit{conditioned} on a history. Each row of the Hankel matrix is the same as the SDM except scaled by a constant, the likelihood of the history that indexes the row \cite{baconLearningPlanning2015}.}

The forward matrix $\forw$ is an infinite-by-$S$ matrix that expresses the joint likelihood of observing a history and the current state.
After choosing a suitable ordering to map histories to indices (histories of length one, then histories of length two, and so on), we may write $\forw_{hist, i} = p(hist, s_{t+1}=s^i)$ for a given history $\hist = (a_1, o_1, \dots, a_t, o_t)$.
In terms of the POMDP matrices $\mc T$ and $\mc Z$, a row of the forward matrix can be expressed by right-multiplying the corresponding matrices in the order of the history:
\begin{equation}\label{eqn:forward-derivation}
    \forw_{\hist, :} = b_0^T \cdot O^{a_1,o_1} T^{a_1} \cdots O^{a_t, o_t} T^{a_t}.
\end{equation}
The backward matrix $\back$ is an $S$-by-infinite matrix that represents the \textit{conditional} likelihood of obtaining the observations of a test after executing its actions, e.g. $\back_{i, \ttest} = p(test| s_{t+1}=s^i)$. If the $\ttest = (a_{t+1}, o_{t+1}, \dots, a_n, o_n)$, then a column of $\back$ can also be computed similarly to the rows of $\forw$:
\begin{equation}\label{eqn:back-derivation}
    \back_{:, \ttest} = O^{a_{t+1}, o_{t+1}} T^{a_{t+1}} \cdots O^{a_n, o_n} T^{a_n} \cdot \ones,
\end{equation}
where $\ones$ is a vector of one in all entries.

The product of the forward and backward matrices results in the Hankel matrix, which we denote as $\HH$.
The matrix multiplication unconditions and then marginalizes out the intermediate hidden state.
Given a history $\hist = (a_1, o_1, \dots, a_t, o_t)$ and a test $\ttest = (a_{t+1}, o_{t+1}, \dots, a_n, o_n)$, a Hankel matrix entry is the corresponding joint likelihood of receiving a full string of observations conditioned on taking a full string of actions, e.g.
\begin{equation}\label{eqn:hankel-entry}
   \HH_{\hist, \ttest} = \PP(o_1, \dots, o_n | a_1, \dots a_n).
\end{equation}

The Hankel matrix does not refer to the underlying hidden state of the POMDP and can be estimated from action-observation trajectories.
If we had a long string of actions and observations $\mc D_n = (a_1, o_1, \dots, a_n, o_n)$, the matrix $\HH$ could be estimated by the \textit{suffix-history approach}, taking frequency counts of subsequences of increasing lengths \cite{wolfeLearningPredictive2005,bootsClosingLearningplanning2011}:
\begin{equation}\label{eqn:emp-hankel-est}
    \HHhat_{\hist, \ttest} = \frac{
        \sum_{i=1}^{n - L} \mb I_{(a_i,o_i,\dots,a_{i+L},o_{i+L}) = \hist \oplus \ttest}
    }{
        \sum_{i=1}^{n - L} \mb I_{(a_i,\dots,a_{i+L},) = \acts{(\hist \oplus \ttest)}}
    }
\end{equation}
 where $\acts(\hist \oplus \ttest)$ is the action sequence associated with $\hist \oplus \ttest$, $\oplus$ is the concatenation operator, and $L = |\hist \oplus \ttest| < n$.

 It is important to note that expressing the Hankel matrix as a factorization of $\forw$ and $\back$ represents the system under a memoryless policy where future actions are independent of previous observations.
To correctly estimate the matrix via \cref{eqn:emp-hankel-est}, the data must also be collected under a memoryless policy.
A uniformly random exploration policy admits this condition, but a larger class of non-memoryless exploration policies can be used for importance sampling~\cite{bowlingLearningPredictive2006}.

\vspace{-.15in}
\subsection{Assumptions}\label{sec:assumptions}
\vspace{-.1in}


Before we discuss our algorithm, we must state some key assumptions.
First, we assume that under a memoryless random exploration-policy $a \sim \pi_{\text{exp}}(\mc A), \pi \in \Delta(\mc A)$\footnote{$\Delta(\mc A)$ denotes the set of distributions over the discrete set $\mc A$.} (which, in this paper, we take to be uniform), the induced Markov chain  $(s_t, a_t, o_t)_{t \geq 0}$ is ergodic. 
This property causes the visitation distribution over states to converge to a stationary distribution $b_\pi$, with nonzero support over the state space, as the agent explores. 
Ergodic Markov chains are \textit{irreducible}, i.e., every state is reachable with nonzero probability, and \textit{aperiodic}, i.e., cannot be trapped in periodic cycles.
We believe these conditions are reasonable. 
If an agent becomes trapped in some connected component of the system, then the dynamics of that component are the ones that are learned. 
Robots also often have passive observation that do not change the environment state. 
The transition dynamics of these actions are loops with period one, which breaks periodicity.

Second, we also assume that $\forw$, when limited to indices corresponding to one fewer than the maximum sequence length, has the same rank as the number of states (e.g. is full-rank), and that $\back$ is also full-rank.
This assumption is strictly weaker than prior work that assumes that all transition matrices $T^a$ and row-stochastic observation matrices $\Obs^a$ are full rank for all actions \cite{azizzadenesheliReinforcementLearning2016,guoPACRL2016}.
The implication can be proven by showing that the products $\diag(b_\pi)T^a \Obs^a \diag(b_\pi)^{-1}$ and $T^a \Obs^a$ are submatrices of $\forw$ and $\back$, respectively.
If $b_\pi$ is a stationary distribution, as implied by our ergodicity assumption, then every entry is positive, so $\forw$ and $\back$ are full-rank.
Many of the standard POMDPs in the literature, which are also included in our experiments (Sec. \ref{sec:experiments}), are counterexamples for the converse of the implication.


Our assumptions have a few consequences on the estimated Hankel matrix.
The starting state distribution $b_0$ at the start of the learning problem, so $b_0$ has little influence over the Hankel matrix. 
Instead, the Hankel matrix will take on the stationary distribution $b_\pi$ as the initial distribution instead.
Furthermore, the rank of the resulting Hankel matrix will be \textit{equivalent} to the number of states of the POMDPs in the restricted class that adhere to our assumptions, as opposed to the lower bound as is the case for general POMDPs \cite{singhPredictiveState2004,huangMinimalRealization2016}. 
A proof of these claims is in Appendix \ref{sec:proof-assumption-consequences} in the supplemental material.\footnote{Supplemental material is made available at \texttt{https://sageshoyu.github.io/\\assets/wafr-2026-supp.pdf}.}




\vspace{-.1in}
\subsection{Constructing a Linear Predictive State Representation}\label{sec:psrs}
\vspace{-.05in}

Suppose we have a Hankel matrix $\HH$, estimated in the limit of infinite data.
Since an implication of Lemma \ref{lem:assumptions-consequences} is that the rank of Hankel matrix $\HH$ is equivalent to the number of states of the POMDP, a natural first step of our method (and learning a PSR) is to compute a rank factorization of $\HH$ \cite{bootsClosingLearningplanning2011,balleSpectralLearning2014}.
One way to achieve this factorization is to compute a singular-value decomposition of the Hankel matrix $\HH = U \Sigma V^T$, where singular values under a specified threshold (and their corresponding orthogonal vector components) are dropped.
The truncated SVD is converted into a \textit{rank factorization} by computing $A = U \Sigma$ to be the left factor and $V^T$ to be the right factor.
Crucially, since $A \cdot V^T$ and $\forw \cdot \back$ both form rank factorizations of $\HH$ (according to assumptions in Sec \ref{sec:assumptions}), there must exist some invertible transformation $P$ such that $A = \forw \cdot P^{-1}$ and $P \cdot \back = V^T$ (see Appendix~\ref{sec:appendix-proof-sim-trans} in the supplemental material).

Moving one step earlier in the Hankel construction (Sec. \ref{sec:hankel-construction}), we can relate transitions, observations, and initial distributions  with the rank factors and the Hankel matrix using \cref{eqn:forward-derivation,eqn:back-derivation}.
Let $\hists^{ao}$ denote an ordered set of all history indices that end in action-observation pair $ao$, and $\hists^{-ao}$ denote the same set with the same ordering but without the ending pair $ao$.
From \cref{eqn:forward-derivation,eqn:back-derivation}, we observe for each $a \in \mc A, o \in \mc O$:
\begin{align}
    \HH_{\hists^{ao}, :} & = \forw_{\hists^{-ao}, :} \cdot O^{ao} T^a \cdot \back = A_{\hists^{-ao}, :}  (P O^{ao}T^a P^{-1})  V^T\label{eqn:trans-obs-solve} \\
    \HH_{\varepsilon, :} & = b_0^T \cdot \back = (b_0^T P^{-1}) \cdot V^T\label{eqn:init-vec-solve}                                                                          \\
    \HH_{:, \varepsilon} & = \forw \cdot \mathbf 1 = A \cdot (P \mathbf 1)\label{eqn:final-vec-solve}
\end{align}
After applying the Moore-Penrose inverse of $A$ and $V^T$ to \crefrange{eqn:trans-obs-solve}{eqn:final-vec-solve}, we obtain the observation-transition product, initial belief, and final summation vector up to a similarity transform.
The transformed initial belief $m_0^T = b_0^T P^{-1}$ from \cref{eqn:init-vec-solve} is called the \textit{initial vector} and the transformed summation vector $m_{\infty} = P \mathbf 1$ from \cref{eqn:final-vec-solve} is called the \textit{final vector}.
The product $M^{ao} = P O^{ao} T^a P^{-1}$ from \cref{eqn:trans-obs-solve} is called a linear PSR update matrix.
Together, this collection of matrices and vectors forms a \textit{linear PSR model} \cite{littmanPredictiveRepresentations2001,bootsClosingLearningplanning2011}.
A PSR can be used to compute the likelihood of observations $o_1, o_2, \dots, o_n$ under actions $a_1, \dots, a_n$ by computing the product $\PP(o_1, \dots, o_n | a_1, \dots, a_n) = m_0^T  M^{a_1o_1} \cdots M^{a_no_n} m_\infty = b_0^T P^{-1} \cdot P O^{a_1o_1}T^{a_1} P^{-1} \cdots P O^{a_no_n}T^{a_n} P^{-1} \cdot P \mathbf 1$, with appropriate normalizations for conditional calculations.
With a few more details, the argument sketched above is a proof of a result of Carlyle and Paz \cite{carlyleRealizationsStochastic1971a}.
The original result given by the authors was for probabilistic automata, and later versions were proved for HMMs~\cite{hsuSpectralAlgorithm2012,balleSpectralLearning2014,vidyasagarCompleteRealization2011,huangMinimalRealization2016}.

\begin{proposition}\label{prop:prob-machine-off-by-transform}
    Let $\HH = A V^T$  be a rank factorization of a Hankel matrix $\HH$ with $\rank(\HH) = r$ formed from a POMDP with initial state $b_{\pi}$, transition matrices $\{T^a\}$ and diagonal observation matrices $\{O^{ao}\}$.
    Suppose $m_0, \{M^{ao}, \forall (a, o) \in \mc A \times \mc O\}, m_\infty$ are computed as in \cref{eqn:trans-obs-solve,eqn:init-vec-solve,eqn:final-vec-solve}.
    Then there exists a nonsingular matrix $P \in \mathbb R^{r \times r}$ such that $P^{-1} M^{ao} P = O^{ao}T^a$ for all $a \in \mc A, o \in \mc O$, $m_0^T P = b_\pi^T$, and $P^{-1} m_\infty = \mathbf 1$.
\end{proposition}

\vspace{-.25in}
\section{Computing the Similarity Transform}\label{sec:method}
\vspace{-.1in}
When a reward is part of the observation, the PSR representation $M^{ao}$ can be used to compute policies that maximize the expected reward. However, if we had access to $O^{ao}$ and $T^a$, we could compute policies for any reward function. To recover $O^{ao}$ and $T^a$, we need a way to recover the similarity transform $P$ and apply Prop. \ref{prop:prob-machine-off-by-transform}.
The observation and transition matrices can be computed from products $O^{ao} T^a$ by taking the sums of the rows to form the diagonal of $O^{ao}$ and then normalizing to form $T^a$.
In many problem scenarios, $P$ is not the identity matrix,
since the rank factors computed via SVD are orthogonal matrices, which is not always the case for $\forw$ and $\back$.
Our approach can recover $P$ up to a certain partition of states, which we introduce with an example. 

\vspace{-0.1in}
\subsubsection{A running example.} Consider the POMDP illustrated in \cref{fig:sense-float-reset}, which is modified from the Float-Reset domain introduced by Littman and Sutton \cite{littmanPredictiveRepresentations2001}.
Like the original, the \texttt{float} action transitions the state up and down a line graph and will always emits an observation of \texttt{0}.
The \texttt{reset} action, also identical to the original, deterministically sets the state to the left end of the graph.
This action emits an observation of \texttt{1} if the state is already in the leftmost state and \texttt{0} otherwise.
The observations of the \texttt{sense} action are the same as \texttt{reset}, except each state of the system does not change.
We also augment the system with a reward function; the agent obtains +1 reward for executing any action in the state adjacent to the reset state and zero reward otherwise.
This system is challenging to learn due to its nontrivial state aliasing.
Aside from the two leftmost states (when treating rewards as observations), all other states in this POMDP have the same observation distributions, regardless of the action.

We wish to capture the difficulties of Sense-Float-Reset to discuss the main output of our algorithm in general terms.
For arbitrary POMDPs, we group states that have the same observation distribution to form a \textit{partition} of states.
We call this grouping an \textit{observability partition}.
Of particular importance is the collection of observation distributions that correspond to actions with full-rank transition matrices.

\begin{definition}[Full-Rank Observability Partition]
    Let $\mc S_\Pi \subset 2^{\mc S}$ be a partition of states, such that for any set $S \in S_\Pi$, states $s^i, s^j \in S$ if and only if $\Obs^{a}_{i, :} = \Obs^a_{j, :}$ for all $a \in \Afull$. 
    We call $\mc S_\Pi$ a full-rank observability partition.
\end{definition}

\vspace{-0.3in}
\subsection{Recovery up to a Full-Rank Observability Partition}\label{sec:main-result}
\vspace{-0.1in}

\begin{figure}[t]
    \centering
    \includegraphics[width=\linewidth]{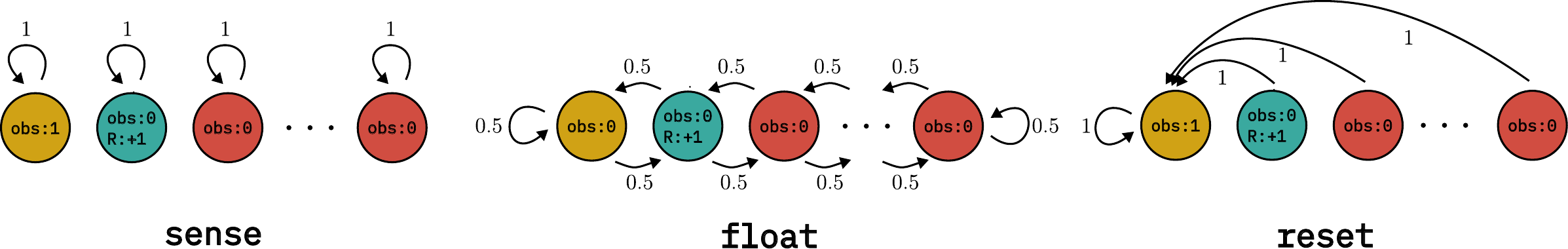}
    \caption{Sense-Float-Reset. Edges are labeled with transition probabilities, and nodes are labeled with observations and reward received upon \textit{leaving} the state. The reward of a state is zero unless specified otherwise. Node shade denotes observability partitions.}\label{fig:sense-float-reset}
    \vspace{-0.50cm}
\end{figure}
Our algorithm can estimate the similarity transform $P$ up to the \textit{full-rank observability partition}, which we formalize in Theorem \ref{thm:coarse-grain-marg-and-trans}.
Our statement is given in the regime of infinite data; for parameters introduced for finite data, see Appendix~\ref{sec:finite-data-params} in the supplemental material.

\begin{theorem}\label{thm:coarse-grain-marg-and-trans}
    Let $\HH$ be a Hankel matrix of POMDP $(\mc S, \mc T, \mc A, \mc O, \mc Z, b_{\pi}, R, \gamma)$ that adheres to the assumptions in Sec. \ref{sec:assumptions}, where $b_\pi$ is the stationary distribution under exploration policy $a_t \sim \Unif(\mc A)$.
    Let $S_{\Pi} \subset 2^{\mc S}$ be the full-rank observability partition of the POMDP.
    Let $m_0$, $\{M^{ao}: a \in \mc A, o \in \mc O\}$, and $m_\infty$ be the linear PSR model as computed via \crefrange{eqn:trans-obs-solve}{eqn:final-vec-solve}.
    Then there exists an algorithm on $m_0$, $\{M^{ao}\}$, and $m_\infty$ that computes a nonsingular matrix $\tilde P$, such that if   
    \begin{align}
         & \tilde b_\pi^T = m_0^T \tilde P = b_\pi^T P^{-1} \tilde P\label{eqn:part_b_pi}                                                                                                                               \\
         & \tilde O^{ao} \tilde T^a = \tilde P^{-1} M^{ao} \tilde P = \tilde P^{-1} P O^{ao}T^a P^{-1} \tilde P \label{eqn:part_paz}                                                                                          \\
         & \tilde b_{\infty} = \tilde P^{-1} m_\infty = \tilde P^{-1} P \mathbf 1\label{eqn:part_b_inf}           \\          
        \intertext{then}
         & \sum_{s^i \in S} \tilde {b_{\pi}}_{i} = \sum_{s^i \in S} {b_{\pi}}_i\label{eqn:sum_psr_vec_parts}                                                                                                                 \\
         & \sum_{s^i \in S} (\tilde b_\pi^T \tilde O^{a_1,o_1}\tilde T^{a_1}  \dots \tilde O^{a_n,o_n} \tilde T^{a_n})_i = \sum_{s^i \in S} (b_\pi^T  O^{a_1o_1}T^{a_1} \dots O^{a_no_n} T^{a_n})_i\label{eqn:sum_apply_paz} \\
         & \tilde b_{\infty} = \mathbf 1 \label{eqn:obtain_sum_vector}
    \end{align}
    for all $a_1, \dots, a_n \in \mc A$, $o_1, \dots, o_n \in \mc O$, integer $n > 0$ and partition set $S \in S_{\Pi}$.
\end{theorem}

\begin{figure}[t]
    \centering
    \includegraphics[width=\linewidth]{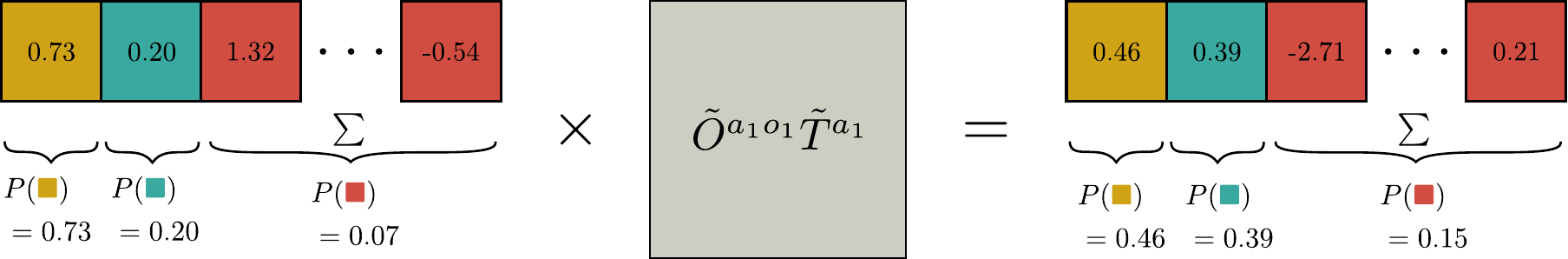}
    \caption{An illustration of Theorem \ref{thm:coarse-grain-marg-and-trans} applied to Sense-Float-Reset. While the individual values inside a single box (state) may not be read as the belief likelihood of the system occupying that state, summing indices over partitions, represented by box shades, will compute the likelihood of the system state in that partition.}\label{fig:partition-calc}
    \vspace{-0.5cm}
\end{figure}

What Theorem \ref{thm:coarse-grain-marg-and-trans} states is that we must sum over indices of the initial `belief vector' to compute the likelihood the system is in a particular partition (\cref{eqn:part_b_pi}).
The same remains true when computing joint likelihoods between observations and the current state partition (\cref{eqn:sum_apply_paz}; see \cref{fig:partition-calc} for a worked example for Sense-Float-Reset).
For POMDPs that have unique observation distributions across all actions, each state is in its own singleton partition, and we can recover the full similarity transform. Otherwise, we recover $P$ up to the full-rank observability partition.
We note it is possible for us to recover some POMDPs that have fewer observations than states, since the collection of \textit{distributions} over emitted observations across all actions must be distinct (see Appendix \ref{sec:appendix-noisy-hallway} in the supplemental material for examples).


To benefit from the result of Theorem \ref{thm:coarse-grain-marg-and-trans}, the systems to be learned must satisfy the assumptions discussed in Sec. \ref{sec:assumptions} and contain full-rank actions.
Full-rank actions commonly arise in many automated manipulation contexts, our domain of interest.
In automated manipulation, robot actions have a desired transition state but may also \textit{fail} (a gripper misses a grasp, slips of a drawer handle, etc.).
One way these actions have been modeled in robot planning systems is to designate a successful `desired state' with some success likelihood $p_{succ}$, and have the system state `fail' with some likelihood (causing a self-transition) \cite{kaelbling_integrated_2013,garrett_online_2020}.
In POMDP terms, these types of actions can be simply modeled as the convex combination $p_{\text{succ}} T + (1 - p_{\text{succ}}) I$, where $T$ is a matrix with rows containing all zeros except for a single entry of $1$ (the `desired states'), the identity $I$ indicates self-loop failure dynamics, and $p_{succ}$ the likelihood of an action succeeding.
Under mild assumptions (e.g. $p_{\text{succ}} \neq 1/2, 1$), these actions are full-rank (see Appendix~\ref{sec:appendix-full-rank-trans} in the supplemental material).

\vspace{-.15in}
\subsection{Recovering Observation Distributions from Full-Rank \vspace{-.1in}
Actions}\label{sec:obs-as-eigenvalues}

We now introduce an algorithm that computes the similarity transform $\tilde P$.
Our approach is a reformulation of the tensor decomposition method \cite{anandkumarTensorDecompositions2012,azizzadenesheliReinforcementLearning2016} for linear PSR models.
Our procedure begins by marginalizing out the observations in matrices $M^{ao}$, yielding the transitions $T^a$ up to similarity transform $P$.
This marginalization can be done by summing all matrices $M^{ao}$ over all $o \in \mc O$ for some fixed $a \in \mc A$: 
\begin{equation}\label{eqn:marg-out-obs}
    \sum_{o \in \mc O} M^{ao}
    = P \bigg(\sum_{o \in \mc O} O^{ao} T^a \bigg) P^{-1}
    = P T^a P^{-1}
\end{equation}
With a slight abuse of notation, we denote $P T^a P^{-1}$ as the matrix $M^a$.
The next step of our procedure continues with transitions that are full-rank, which can easily be determined by a threshold test on the singular value decomposition on all matrices $M^a$.
Let $\mc M_{\full} = \{P T^a P^{-1} : a \in \Afull \}$ be the set of full-rank transitions.
Next, we compute the diagonal observation matrices associated with the full-rank actions.
For each $M^a \in M_{\full}$ and $o \in \mc O$ we compute
\begin{equation}\label{eqn:diag-matrix-prep}
    M^{ao} \cdot {M^a}^{-1}
    = P O^{ao} T^a P^{-1} (P T^a P^{-1})^{-1}
    = P O^{ao} P^{-1}.
\end{equation}
Since we know that all matrices $O^{ao}$ are diagonal, the eigenvalues of the matrices $M^{ao}{M^a}^{-1}$ will be the diagonal entries of $O^{ao}$.
If the entries of a particular $O^{ao}$ are unique, then the eigenvectors computed from an eigendecomposition of $M^{ao} {M^a}^{-1}$ will recover the columns of $P$ up to a scalar factor.
However, it is common to have repeated observation likelihoods across states for a single action (like all of Sense-Float-Reset), and an eigendecomposition may produce \textit{any} spanning set of the invariant space corresponding to the repeated eigenvalue.


To reduce ambiguity, we wish to compute a \textit{joint diagonalization} of all matrices $M^{ao} {M^a}^{-1}$, which attempts to diagonalize each matrix with the same similarity transform.
We apply a method of He et al. \cite{heRandomizedMethods2024}.
Their method exploits the fact that \textit{sums} of matrices $\{\simobs : a \in \Afull, o \in \mc O\}$ do not change the invariant spaces spanned by eigenvectors of each matrix $\simobs$.
Suppose $\{w^{ao}: a \in \mc \Afull, o \in \mc O\}$ is a set of weights, then the weighted sum 
\begin{equation}\label{eqn:rand-sim-obs-sum}
    \sum_{a \in \Afull, o \in \mc O} w^{ao} \simobs
    = P \bigg(\sum_{a \in \Afull, o \in \mc O} w^{ao} O^{ao}\bigg) P^{-1}
\end{equation}
is still diagonalizable by $P$.
Should we choose \textit{random} weights $w_{ao}$, then the eigenvalues will be distinct up to states that share the same observation distribution almost surely.
We sample these weights from the unit sphere $\mb S^{|\Afull|\cdot|\mc O| -1}$~\cite{heRandomizedMethods2024}.
\begin{lemma}\label{lem:unique-obs-unique-eigvals}
    Let weights $\{w_{ao}: a \in \Afull, o \in \mc O\}$ be sampled i.i.d. with respect to $\Unif(\mb S^{|\Afull|\cdot|\mc O| - 1})$ and $\Lambda = \sum_{a \in \Afull, o \in \mc O} w^{ao} O^{ao}$.
    Then $\Lambda_{ii} = \Lambda_{jj}$ with prob. 1 if and only if $O^{ao}_{ii} = O^{ao}_{jj}$ for all $o \in \mc O$ and all $a \in \Afull$.
\end{lemma}
When multiple states have the same observation distribution for all actions, the eigenvalues corresponding to those states will be the same, so their eigenvectors cannot be uniquely determined.
Thus, the similarity transform $P'$ is nonunique when we have a nontrivial full-rank observability partition, the consequences of which we discuss in the next session.

\vspace{-.15in}
\subsection{Recovering Partition-Level Belief Likelihoods and Transitions}\label{sec:trans-partitions}
\vspace{-.1in}

The recovered similarity transform $P'$ formed by the eigenvectors of the random sum in \cref{eqn:rand-sim-obs-sum}, but not the partition-level transitions.
When the full-rank observability partition is nontrivial, the matrix $Q = P^{-1} P'$ is block-diagonal, with invertible blocks that correspond to states within the same partition (see Appendix \ref{sec:proof-recovery-up-to-block-diag} in the supplemental material for a proof). 
This matrix $Q$ prevents us from using $P'$ as the similarity transform promised in Theorem \ref{thm:coarse-grain-marg-and-trans}.
For example, when applying $P'$ as a similarity transform to the PSR vector $m_0$, a restriction to the subindices of the partition $S_1$ yields $[m_0^T P']_{S_1} = [b_0^T P^{-1} P']_{S_1} = [b_0^T]_{S_1} Q_1$, so the sum of the entries is not a proper likelihood, violating \cref{eqn:sum_psr_vec_parts} of Theorem~\ref{thm:coarse-grain-marg-and-trans}.

To recover partition-level likelihoods and transitions, we look to the final vector of the linear PSR after applying the transform $P'$, e.g. $P'^{-1} m_\infty = P'^{-1} P \mathbf{1} = Q^{-1} \mathbf 1$.
Intuitively, by applying $P' \diag(P'^{-1} m_\infty)$ as a similarity transform, we transform the final vector back to $\mathbf{1}$, recapturing a marginalization of the latent state variable.
To avoid scenarios where $P'^{-1} m_\infty$ has entries of zero, we perform a pre-processing step by multiplying the system with a random block-diagonal rotation matrix $R$, whose blocks correspond to the full-rank observability partition. 
We take the transform $P' R \diag(R^T P'^{-1} m_\infty)$  as the similarity transform $\tilde P$ that satisfies Theorem \ref{thm:coarse-grain-marg-and-trans} (see Appendix \ref{sec:appendix-proof-main-result} in the supplemental material for the proof of correctness and runtime).

\bigskip
\vspace{-0.1in}
\begin{figure}[t]
    \centering
    \includegraphics[width=\linewidth]{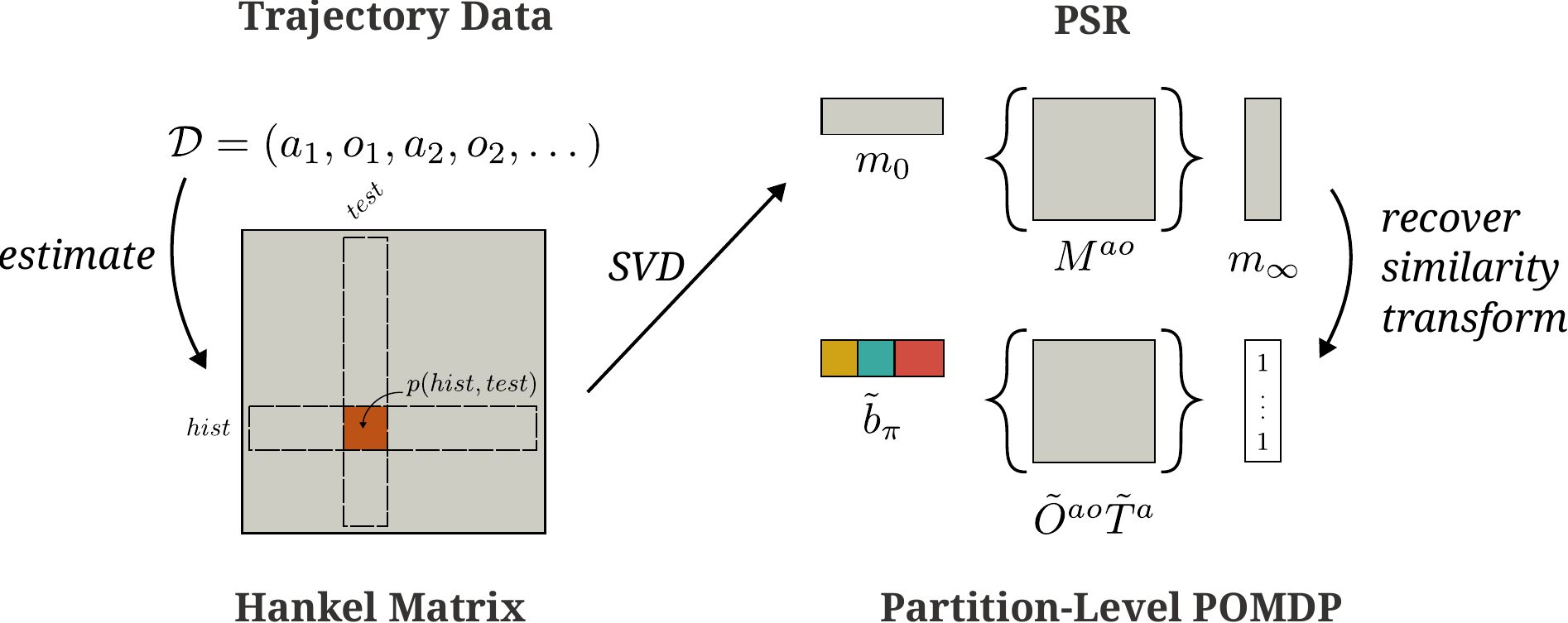}
    \caption{An illustration of the entire algorithm pipeline described from Sec. \ref{sec:hankel-construction}-\ref{sec:method}.}\label{fig:pipeline}
    \vspace{-0.50cm}
\end{figure}
\noindent\textbf{Summary}
A diagram that sketches the entire learning algorithm can be found in \cref{fig:pipeline}.
Prop. \ref{prop:prob-machine-off-by-transform} and Theorem \ref{thm:coarse-grain-marg-and-trans} suggest that we can compute a PSR and then recover the corresponding system model by recovering the unknown similarity transform $P$. 
Sections \ref{sec:obs-as-eigenvalues} and \ref{sec:trans-partitions} present an algorithm that applies to all POMDPs that admit the assumptions of Sec. \ref{sec:assumptions}, and computes $P$ up to the full-rank observability partition of states. 
The set of POMDPs learned by our method captures a large class of real-world POMDPs that arise in manipulation scenarios.

\vspace{-.15in}
\section{Experiments}\label{sec:experiments}
\vspace{-.1in}
\scomment{
    In particular, we seek to answer the questions:

    \begin{enumerate}
        \item Can the transition and observations learned by our method with a finite amount of data be exploited by planning algorithms that require more general inference operations?
        \item Does our `correction' of a PSR reduce the performance of the learned model?
        \item How much data is required to learn a model that yields comparable planning performance to the ground truth?
    \end{enumerate}
}

Our experiments evaluate the fidelity of the learned POMDP models and explore the advantages of estimating transition and observation likelihoods.
We seek to know, empirically, how quickly the learned partition-level transition and observation matrices converge to ground truth values. 
We also wish to know how the performance of the planning model is impaired by estimation error from little data.
Lastly, we evaluate whether the transition and observation likelihood estimates can be leveraged to specify a reward function to elicit desired behavior from a planner.
All experiments are compared against linear PSRs and Expectation-Maximization (EM) \cite{rabinerTutorialHidden1989,shatkayLearningGeometricallyConstrained2002} with a number of states determined by the number of components of the truncated SVD when learning a linear PSR.






\begin{figure}[t]
    \centering
    \includegraphics[width=0.7\linewidth]{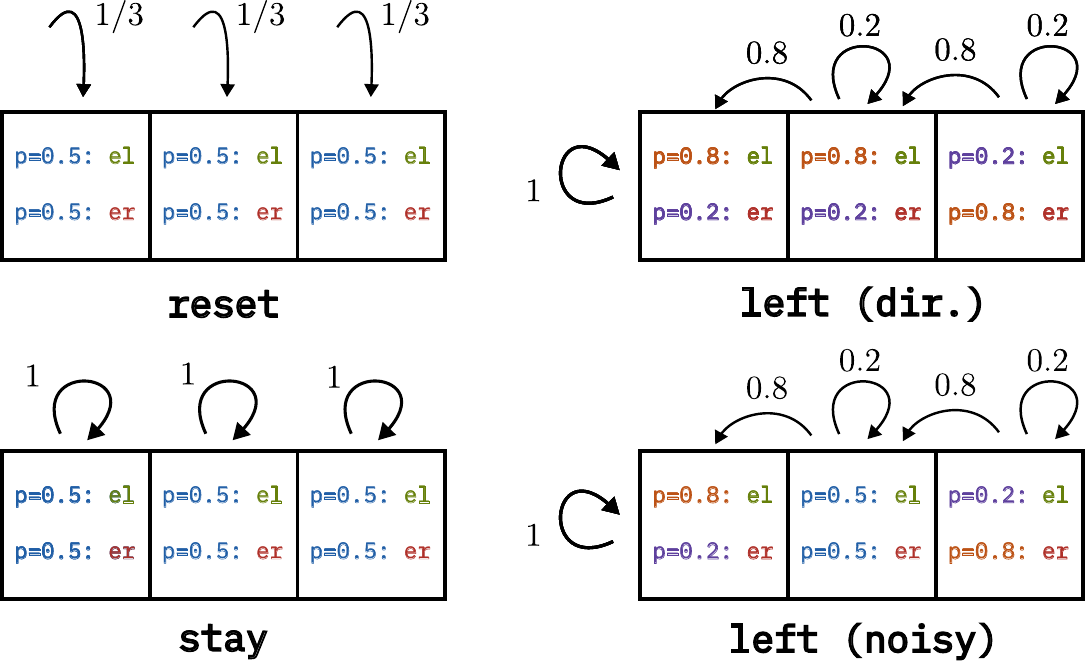}
    \caption{The directional and noisy hallway domains. The observations \ttt{el} and \ttt{er} represent `end-right' and `end-left' respectively. 
    Transitions are shared and observation distsributions differ in the middle state in the \ttt{left} and \ttt{right} actions.
    The \ttt{right} action is the mirror image of \ttt{left}, with the obs. likelihoods swapped in the middle state in the directional environment.}\label{fig:all-hallways}
    \vspace{-0.5cm}
\end{figure}

For our planning experiments, we verify our approach on two standard domains, Tiger \cite{kaelbling_planning_1998} and T-Maze (with a single corridor state)~\cite{bakkerReinforcementLearning2001}, and Sense-Float-Reset.
To allow the agent to collect an arbitrary-length string of data in all domains, we modify T-Maze to choose the next state randomly from the initial state distribution instead of terminating the sequence of interactions. 
In the supplemental material, Appendices \ref{sec:finite-data-params} and \ref{sec:appendix-exps} contain details on the parameters of the learning algorithm and planner, including sensitivity analyses to algorithm parameters. 
Rewards of the original POMDPs have been learned as observations for planning.
For our reward-specification experiments, we introduce two novel domains (\textit{noisy hallway} and \textit{directional hallway}) whose observation and transition matrices can be fully recovered by our method (see \cref{fig:all-hallways}).
For more details on all domains, see Appendix \ref{sec:appendix-noisy-hallway} in the supplemental material.

\begin{figure}[t]
    \centering
    \includegraphics[width=\linewidth]{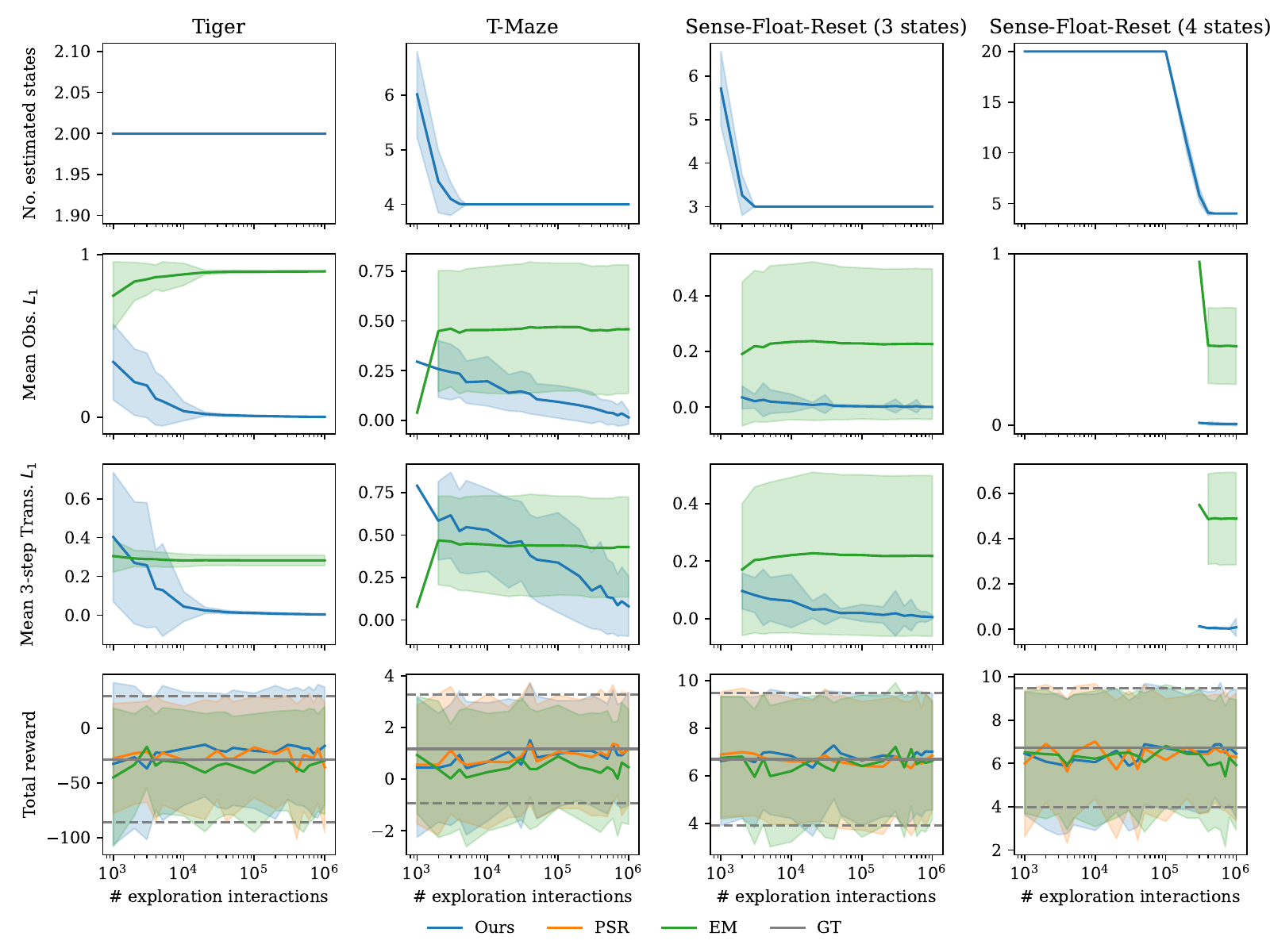}
    \vspace{-0.75cm}
    \caption{
    Curves plot means and error bars are standard deviation over 100 seeds. 
    \textbf{Row~1:} Estimated number of states.
    \textbf{Row 2:} Obs. matrix error relative to ground truth. 
   \textbf{Row~3:} Trans. error from belief state with observations marginalized after three action sequences. 
    This error is only measurable once the estimated number of states matches that of ground truth, which truncates the curves. Observation and three-step transition distribution parameter estimates are projected to the probability simplex before comparison with GT.
    \textbf{Row 4:} Total reward from planner under different models.}
    \label{fig:all-plots}
    \vspace{-0.50cm}
\end{figure}

\vspace{-0.2in}
\subsubsection{Convergence to true POMDP parameters.} In \cref{fig:all-plots}, our results suggest that our method successfully recovers the underlying observation models through the $L_1$ error of learned observation and three-step partition-level belief vectors (after observations are marginalized) against ground truth.
EM consistently converges to a local minimum that correctly predicts future observations but obtains incorrect POMDP matrices.


\vspace{-0.2in}
\subsubsection{Planning performance with the learned model.}
To evaluate the performance of the learned models, we apply a standard solver to the original ground truth POMDPs, learned PSRs, and learned POMDPs.
We use the sampling-based planning approach PO-UCT of Silver and Veness \cite{silver_montecarlo_2010} with the correction described by Shah et al. \cite{shahNonasymptoticAnalysis2022}.
Ideally, average planning performance should be the same across ground truth models, PSRs, and the partition-level POMDPs, since they all learn the same distribution over observations and rewards given a potential sequence of actions
(see Appendix \ref{sec:appendix-sampling-strategies} in the supplemental material for rollout strategies for each model).
\Cref{fig:all-plots} reports that the performance across models is roughly the same.

\vspace{-0.2in}
\subsubsection{Planning performance on specified rewards.}
We explore whether the likelihoods and observations yielded by our algorithm can be leveraged to direct agent behavior by specifying a reward function.
When explicit POMDPs are available, we can analyze the learned observation matrices to directly find the states to assign positive reward.
In the past, if a PSR did not learn a reward model, then rewards were determined by observations \cite{bootsClosingLearningplanning2011}.
Otherwise, the entire model must be relearned to estimate a reward model that depends on state \cite{izadiPointBasedPlanning2008}.

\begin{figure}[t]
    \centering
    \includegraphics[width=\linewidth]{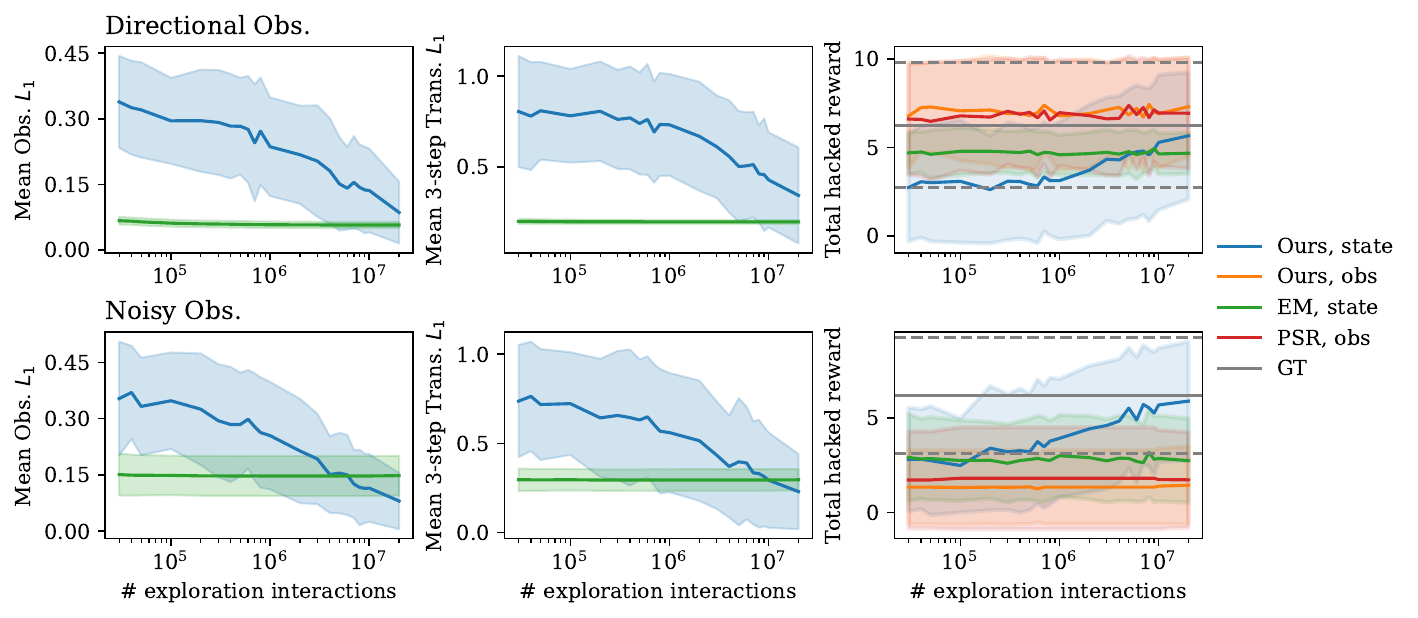}
    \caption{`Obs' refers to assigning rewards to action-observation pairs; `state' refers to assigning rewards to states. Curves plot means and error bars are standard deviations over 100 seeds. `Hacked reward' refers to planning performance under new reward fun. Observation and three-step transition distribution parameter estimates are projected to the probability simplex before comparison with GT.}
    \label{fig:swap-reward}
    \vspace{-0.50cm}
\end{figure}

Our evaluations of this experiment are carried out on the two noisy hallway domains, where we attempt to direct the agent to drive the POMDP to the `middle' hallway state with ambiguous observations (say, to collect more data for a poorly functioning perception model).
We compare the strategies of assigning rewards to observations and assigning rewards to states.
In the directional domain, we assign $+1$ reward to action-observation pairs $(\ttt{left}, \ttt{end-left})$ and $(\ttt{right}, \ttt{end-right})$ for the former strategy, and assign $+1$ reward to the state whose maximum likelihood observations under \texttt{left} and \texttt{right} are their corresponding hallway ends for the latter.
For the noisy environment, we reward the same action-observation pairs as the directional environment and also $(\ttt{left}, \ttt{end-right})$ and $(\ttt{right}, \ttt{end-left})$ for the former strategy, and add $+1$ reward to the state that maximizes the sum of entropy of observation distributions across all actions for the latter. 
The former strategy is evaluated on PSRs and POMDPs, whereas the latter is evaluated on learned POMDP models only.
Performance is judged on total reward gathered under the new reward function.

Results can be found in \cref{fig:swap-reward}.
In the directional domain, models that use the first strategy allow the planner to drive the middle state because it is easily identified by the observations received under \ttt{left} and \ttt{right}. 
The second strategy performs poorly with less data due to slow convergence of transition matrices (see Appendix \ref{sec:appendix-transition-convergence} in the supplemental material).
In the noisy domain, the uniform belief state and belief state that places all mass on the middle of the hallway yield the same observation mixture distribution when weighted by the belief over states.
The planner that uses observation-based rewards sees that playing the \ttt{reset} action or controlling the state to stay in the center to yield the same rewards, leading to unnecessary \ttt{reset} actions.
The planner that uses the rewards emitted from the highest-entropy state correctly rewards the middle state after the transition matrices begin to converge, eliciting the correct behavior.
This additional flexibility highlights that learning POMDPs maintains all the advantages of PSRs and obtains the flexibility to exploit observation and transition likelihood models.


\vspace{-.1in}
\section{On the Necessity of Observability 
\vspace{-.1in}
Partitions}\label{sec:info_impossibility}
While learning POMDPs up to observability partitions falls short of our goal of recovering all possible POMDPs with arbitrary transition and observation matrices, the restriction of the observability partitions is in fact the best we can achieve from a single sequential trajectory.
We will show that there are multiple POMDPs with different dynamics within the observability partitions that yield the same distribution over observations regardless of the actions taken, even when we assume that $\forw$ and $\back$ are full-rank.
The principal issue is that there are multiple valid `forward-backward' factorizations of the same Hankel matrix.
An equivalent statement is that there is a similarity transform $P$ that converts one valid forward-backward factorization into another.
Our counterexample to identifiability is framed in terms of the last point: we provide a POMDP and a similarity transformation $P$ that, when applied in the manner of Prop. \ref{prop:prob-machine-off-by-transform}, yields a valid $b_0'$, $O'^{ao}$ and $T'^a$ of a different POMDP.

The counterexample is a perturbation of the \textit{sense-float-reset} with three states (\cref{fig:sense-float-reset}).
We perturb the \ttt{float} transition matrix, which we associate with a new action $\ttt{float}'$, to allow for the states to jump between the two ends of the graph.
The \ttt{reset}, \ttt{sense}, and observation matrices remain unchanged.

\begin{equation}\label{eqn:counterexample}
    T^{\ttt{float}'} = \begin{pmatrix}
        .4 & .5 & .1 \\
        .5 & .1 & .4 \\
        0  & .5 & .5 \\
    \end{pmatrix}
    \quad
    P =  \begin{pmatrix}
        1 & 0   & 0   \\
        0 & 0   & 1   \\
        0 & .95 & .05 \\
    \end{pmatrix}
\end{equation}
It can be readily verified that after applying the matrix $P$ as a similarity transform on all POMDP matrices of the perturbed sense-float-reset problem, that we obtain another valid POMDP.
Under a random exploration policy, the Hankel matrices computed from data trajectories of both POMDPs are the same, in the limit of infinite data (see Appendix \ref{sec:proof_unid} in the supplemental material for proof).
\begin{theorem}\label{thm:same_hankel}
    Let $\mc M = (b_0, \mc A, \mc T, \mc O, \mc Z)$ and $\mc M' = (b_0', \mc A', \mc T', \mc O', \mc Z')$ be the two POMDPs corresponding to the POMDPs as presented above, with arbitrary initial distributions $b_0$ and $b_0'$.
    Suppose the Hankel matrix $\HH = \lim_{n \to \infty} \hat \HH$, where $\hat \HH$ is computed via \cref{eqn:emp-hankel-est} under POMDP $\mc M$, and suppose $\HH'$ is the analogous Hankel matrix under $\mc M'$.
    Then $\HH$ and $\HH'$ are equivalent.
\end{theorem}

We can conclude that the distribution over observations is the same, regardless of the action sequences taken in either POMDP, since these observation distributions can be read from the row of either Hankel matrix indexed by the empty history.
What our counterexample suggests is that knowledge of the internal structure of a POMDP is \textit{not necessary} to predict future outcomes of the system.
While being able to predict future trajectories may be sufficient for planning (with goals set as a function of observations), the ambiguity of internal structure prevents us from fine-grain model manipulation such as setting goals on specific latent states for planning.
Learning exact models that can be leveraged in this way can only be done when certain conditions are met.


\vspace{-.15in}
\section{Related Work}\label{sec:related_work}
\vspace{-.1in}




Much of the theory surrounding Hankel matrix methods has been developed to study Hidden Markov Models (HMMs), which can be viewed as POMDPs with a single action.
Hankel factorization approaches to learn HMMs have been rediscovered many times \cite{itoIdentifiabilityHidden1992,balleSpectralLearning2014,jaegerDiscretetimeDiscretevalued1998}.
Since then, conditions under which the Hankel matrix represents a finite-state HMM have been investigated \cite{andersonRealizationProblem1999,vidyasagarCompleteRealization2011}.
Sample complexity and runtime complexity of these methods have also been studied in many contexts, including from sample trajectories of the HMM \cite{huangMinimalRealization2016,sharanLearningOvercomplete2017} and from directly querying Hankel entries~\cite{mahajanLearningHidden2023}. 

The study of POMDP-learning has been informed by the development of Hankel methods. 
PSRs \cite{littmanPredictiveRepresentations2001,singhPredictiveState2004,bootsClosingLearningplanning2011,wolfeLearningPredictive2005} were discovered separately, but share their mathematical basis with Hankel factorization approaches to Hidden Markov Models above.
In the learning theory community, these spectral approaches have largely been focused on learning predictive models for model-based reinforcement learning \cite{jinSampleEfficientReinforcement2020,chenLowerBounds2023,liuWhenPartially2022,chenPartiallyObservable2022}, but not learning explicit transition and observation distributions. 
The study of \textit{exact} HMM and POMDP recovery has largely been done using tensor decomposition methods, which require full-rankness assumptions on transition and row-stochastic observation matrices \cite{mosselLearningNonsingular2006,huangMinimalRealization2016,azizzadenesheliReinforcementLearning2016,guoPACRL2016}. 
Relaxing these assumptions is the focus of this paper.

Other approaches to learning POMDPs (and HMMs) have also been explored.
Unlike the stochastic case, POMDPs have been shown to be learnable when actions and observations are deterministic by matching states as equivalence classes of observed histories that produce the same outputs on future actions \cite{angluinLearningRegular1987,rivestInferenceFinite1993,deanInferringFinite1995,brafmanRegularDecision2019,roncaMarkovAbstractions2022}.
The Expectation-Maximation algorithm has been applied to both POMDPs and HMMs \cite{rabinerTutorialHidden1989,shatkayLearningGeometricallyConstrained2002}, but are often stuck in local minima.
Mixed-integer programs to search through automata \cite{toroicarteLearningReward2023} and inductive logic schemes \cite{amirLearningPartially2008} have also been applied.
An alternate view of learning deterministic discrete systems has been studied from the lens of combinatorial filters \cite{sakcakMathematicalCharacterization2024}.
Recurrent deep-learning-based architectures have also been explored~\cite{wangLearningBelief2023,allenMitigatingPartial2024,agarwalPerSimDataEfficient2021}.

\vspace{-.15in}
\section{Conclusion and Future Work}
\label{sec:conclusion}
\vspace{-.15in}
We present a method that learns the parameters of a discrete POMDP from an action-observation sequence gathered under a random exploration policy up to a partition of the state space.
Our approach applies tensor decomposition methods to estimate a similarity that to recover a full POMDP model under a milder set of assumptions relative to prior work. 
In domains where each state has a unique observation distribution aggregated across all full-rank actions, we recover the true POMDP.
Otherwise, we learn the transitions between full-rank observability partitions of the state space.
We also show that we can only distinguish POMDPs up to observability partitions from one sequential data trajectory.

One limitation of our work is the assumption that the forward and backward matrices are full-rank.
In general, it is known that the rank of the Hankel matrix can be strictly smaller than the minimal number of states of an HMM, which is also a POMDP \cite{jaegerDiscretetimeDiscretevalued1998,vidyasagarCompleteRealization2011}.
While we have learned much from Hankel factorization algorithms for POMDP learning, the presence of this rank discrepancy suggests a broader algorithmic framework may be required to better understand the problem.
Another limitation is that the algorithm only learns models from domains with few states due to the size of the requisite Hankel matrix.
Prior work on HMMs characterizes computational efficiency based on the minimum size of the Hankel matrix required to ensure full-rankness \cite{huangMinimalRealization2016,sharanLearningOvercomplete2017}.
Generalizing these results to POMDPs may help us scale our algorithms to larger systems.



\begin{credits}
\subsubsection{\ackname}
We thank Prof. Leslie Pack Kaelbling and Prof. Tom\'as Lozano-P\'erez for their constructive feedback on early versions of the formulation of the problem studied in this paper.
Our use of large language models is solely limited to the completion of routine coding tasks, such as experiment launching and plotting code.
This work was supported by ARL under Grant W911NF-23-2-0012, ONR under Grant N00014-22-1-267, the NSF Graduate Research Fellowship Program under Grant 2141064, and the Siegel Family Quest for Intelligence. 
The views and conclusions contained in this document are those of the authors
and should not be interpreted as representing the official policies, either
expressed or implied, of DEVCOM Army Research Laboratory, the Office of Naval Research, the National
Science Foundation, the U.S. Government, or the Siegel Family Quest for Intelligence.

\subsubsection{\discintname}
The authors have no competing interests to declare that are
relevant to the content of this article. 
\end{credits}
%
%
%
\bibliographystyle{splncs04}
\bibliography{references}

\newpage
\renewcommand{\thesection}{\Alph{section}}
\setcounter{section}{0}

\section{Appendix: Omitted Proofs}\label{sec:appendix-proofs}

\subsection{Formalization of Assumptions Consequences}\label{sec:proof-assumption-consequences}

\begin{lemma}\label{lem:assumptions-consequences}
    Let $(\mc S, \mc T, \mc A, \mc O, \mc Z, b_0, R, \gamma)$ be a POMDP. Suppose that the agent has collected a trajectory $\mc D_n = (a_1,o_1,\dots,a_n,o_n)$ for $n > 0$, where $a_i \sim \Unif(\mc A)$ for all $i$.
    Suppose the POMDP admits the assumptions outlined above.
    Furthermore, let $\HHhat$ be the `empirical' Hankel matrix as computed in \cref{eqn:emp-hankel-est}.
    Consider the Hankel matrix in the `limit of infinite data,' where $\HH = \lim_{n \to \infty} \HHhat$.
    Then $\HH$ is the Hankel matrix of POMDP $(\mc S, \mc T, \mc A, \mc O, \mc Z, b_{\pi}, R, \gamma)$ and $\rank(\HH) = |\mc S|$.
\end{lemma}

Here, we prove the consequences of the assumtions discussed in Section \ref{sec:assumptions}.
For our proof, we rely on a fundamental result on the convergence of ergodic Markov chains to stationary distributions.

\begin{theorem}[Ergodic Theorem, \cite{norrisDiscretetimeMarkov1997}]\label{thm:ergodic}
    Let $T$ be ergodic with stationary distribution $\pi$, and let $b_0$ be any initial distribution. Let $(X_n)_{n \geq 0}$ be a Markov chain with respect to $T$ with initial distribution $b_0$.
    Then, for any bounded function $f: \mc S \to \R$ we have

    $$
        \PP \left(
        \lim_{n \to \infty}
        \frac 1 n \sum_{k=1}^{n-1} f(X_k) = \bar f
        \right) = 1
    $$
    where
    $$
        \bar f = \sum_{i \in \mc S} \pi_i f(i),
    $$
    and $b_\pi$ is the stationary distribution of $T$.
\end{theorem}

We also require knowledge of the stationary distribution of Markov chains created as a `sliding window' of another ergodic Markov chain.

\begin{lemma}\label{lem:n-step-converges}
    Let $(Y_t)_{t \geq 0}$ be a Markov chain with ergodic transition matrix $T$, with stationary distribution $\pi$, over state space $\mc Y = \{1, \dots, k\}$.
    Then the Markov chain $(Y_t, Y_{t+1}, \dots Y_{t+{n-1}})_{t \geq 0}$ is also ergodic, with stationary distribution $\tilde \pi(i_1, \dots i_n) = \pi_{i_0} T_{i_0, i_1} \dots T_{i_{n-2}, i_{n-1}}$.
\end{lemma}
\begin{proof}
    The transitions of this Markov chain can be expressed as
    $$
        \PP(i_{t+n}, \dots, i_{t+1} | i_{t+(n-1)}, \dots, i_t) = T_{i_{t+{n-1}}, i_{t+n}}
    $$
    and zero if the indices do not follow the `sliding window' form above. We verify the stationary distribution claimed in the conclusion of the statement.
    Suppose that $(j_1, \dots, j_n)$ is a state of the Markov chain $(Y_t, Y_{t+1}, Y_{t+n-1})$.
    When we apply the transition likelihoods above, we find that
    \begin{align*}
         & \sum_{i_1, \dots, i_n \in \mc Y} \tilde \pi(i_1, \dots i_n) \PP(Y_{t+1}=j_1, \dots, Y_{t + n}=j_n | Y_t = i_1, \dots, Y_{t + n - 1} = i_n) \\
         & = \sum_{i_1} \pi(i_1) T_{i_1, j_1} \dots, T_{j_{n-1}, j_n}                                                                                 \\
         & = \pi(j_1) T_{j_1, j_2} \dots T_{j_{n-1}, j_n} = \tilde \pi(j_1, \dots, j_n).
    \end{align*}
    The second line applies the definition of a transition above and the proposed stationary distribution, and the third line uses the fact that $\pi$ is the stationary distribution of $T$.
\end{proof}

We now have all the tools we need to prove Lemma \ref{lem:assumptions-consequences}.
\begin{proof}

    \textbf{The Hankel matrix takes on the stationary distribution $b_\pi$ as the stationary distribution}:



    We first consider the stochastic process $X_t = (s_t, a_t, o_t)$, which is a Markov chain due to the factorization structure of a POMDP \cite{thrun_probabilistic_2005}.
    Per our assumptions in Section \ref{sec:assumptions}, we assume that the Markov chain $(X_t)_{t \geq 0}$ over the state space $\mc X = \{(s^i, o^j, a^k) \in \mc S \times \mc O \times \mc A : \PP(o^k | s^i a^k) > 0 \}$ is ergodic.
    Suppose that its stationary distribution is $p_{\pi}$.
    We denote ${p_\pi}(s, a, o)$ to be the likelihood of the Markov chain $(s, a, o)$ under $p_\pi$.
    Of interest is the marginal stationary distribution of the \textit{POMDP} state $s \in \mc S$, under $p_\pi$, which we denote as the vector $b_\pi$, where $(b_\pi)_{i} = \sum_{a \in \mc A, o \in \mc O} {p_{\pi}}(s, a, o)$.

    We observe that for any $(a_0, o_0, s_0, \dots, a_{k-1}, o_{k-1}, s_{k-1}) \in \mc X^{k}$, under Lemma \ref{lem:n-step-converges} and Theorem \ref{thm:ergodic},
    \begin{equation}\label{eqn:n-step-conv}
        \begin{aligned}
             & \lim_{n \to \infty}
            \frac 1 {n - k}
            \sum_{i=k}^n \mb I_{a_0, o_0, s_0, \dots, a_{k-1}, o_{k-1}, s_{k-1} = X_{i-k}, \dots, X_i}                \\
             & =
            {p_{\pi}}(a_0, o_0, s_0) \prod_{i=0}^{k-1} \PP(X_{i+1} = s_{i+1}, o_{i+1}, a_{i+1} | X_i = s_i, o_i, a_i) \\
        \end{aligned},
    \end{equation}
    \textit{almost surely} for any integer $k \geq 1$. Let $\PP$ corresponds to the law induced by the stationary distribution $p_{\pi}$ and transition and observation models of the POMDP.
    We will now use \cref{eqn:n-step-conv} as a way to understand the convergent values of the Hankel matrix.

    Let $\bar{\mathcal{D}_n} = (s_1, a_1, o_1,  \dots, s_n, a_n, o_n)$ be a sequence of the induced Markov chain, and let $\mc D_n = (a_1, o_1, \dots, a_n, o_n)$ be the same dataset with the state variable omitted.
    Let $\hist$ $= (a^{j_1}, o^{k_1}, \dots, a^{j_t}, o^{k_t})$ and $\ttest = (a^{j_{t+1}}, o^{k_{t+1}}, \dots, a^{j_L}, o^{k_L})$ be action-observation sequences, length $L < n$.
    Then, we may evaluate the empirical Hankel matrix $\HHhat$ using \cref{eqn:n-step-conv}.
    \begin{align*}
         & \HHhat_{\hist, \ttest}       \\
         & = \frac{
        \sum_{i=1}^{n - L} \mb I_{(a_i,o_i,\dots,a_{i+L},o_{i+L}) = \hist \oplus \ttest}
        }{
        \sum_{i=1}^{n - L} \mb I_{(a_i,\dots,a_{i+L},) = (a^{j_1}, \dots, a^{j_L})}
        }                               \\
         & =
        \frac{1 / (n - L)}{1 / (n - L)} \\ &
        \quad \cdot
        \frac{
        \sum_{s^{m_1}, \dots, s^{m_L}}
        \sum_{i=1}^{n - L}  \mb I_{(a_i,o_i,s_i\dots,a_{i+L},o_{i+L},s_{i+L}) = (a^{j_1},o^{k_1},s^{m_1}, \dots, a^{j_L},o^{k_L},s^{m_L})}
        }{
        \sum_{o^{k_1}s^{k_1}, \dots, o^{k_L}s^{k_L}}
        \sum_{i=1}^{n - L} \mb I_{(a_i, o_i, s_i, \dots, a_{i+L}, o_{i+L} , s_{i+L}) = (a^{j_1}, o^{k_1}, s^{m_1}, \dots, a^{j_L}, o^{k_L}, s^{m_L})}
        }
    \end{align*}

    Taking the limit of $n$ to infinity on both sides and applying \cref{eqn:n-step-conv} (noting the denominator is nonzero almost surely under a uniform random exploration policy) yields
    \begin{align*}
         & \lim_{n \to \infty} \HHhat_{\hist,\ttest}                                                                                                               \\
         & = \PP(o^{k_1}, \dots, o^{k_L} | a^{j_1}, \dots a^{j_L})                                                                                                 \\
         & = \sum_{s^{m_1}, \dots, s^{m_L}, s^{m_{L+1}}} \PP(s^{m_1}, o^{k_1}, \dots, s^{m_L}, o^{k_L}, s^{m_{L+1}} | a^{j_1}, \dots a^{j_L})                      \\
         & = \sum_{s^{m_1}, \dots, s^{m_L}, s^{m_{L+1}}} \PP(s^{m_1}) \PP(o^{k_1}, s^{m_2} | s^{m_1} a^{m_1}) \cdots \PP(s^{m_{L+1}}, o^{k_L} | s^{m_{L}} a^{j_L}) \\
         & = b_\pi^T O^{a^{j_1}o^{k_1}} T^{a^{j_1}} \cdots O^{a^{j_L}o^{k_L}} T^{a^{j_L}} \cdot \mathbf 1
    \end{align*}

    The second line unmarginalizes the state, and the third line factorizes the full likelihood in terms of the POMDP's transition and observation conditional likelihoods. The last line unpacks the probability law $\PP$ introduced by \cref{eqn:n-step-conv} back into matrix notation.
    We can see that splitting the product for the last line above over $\hist$ and $\ttest$ will reproduce individual rows and columns of $\forw$ and $\back$, respectively.
    Finally, we observe that $\forw$ has taken on the distribution $b_\pi$ as the initial vector in the product.





    \textbf{The Hankel matrix is full-rank}:
    Since we know a submatrix of $\forw$ formed a subselection of rows is full-rank then the full forward matrix $\forw$ is full-rank as well.
    Thus, we know that both of the $\back$ and $\forw$ are full-rank. This means we can find $|S|$ linearly-independent rows of $\back$ and $|S|$ linearly-independent columns of $\forw$, which we assemble into submatrices $K$ and $W$ respectively.
    We observe, then, that $K$ and $W$ are full-rank and square.
    The product $K \cdot W$, then, must also be full-rank and square.
    When we multiply $\forw \cdot \back = \HH$, we observe, then, that $K \cdot W$ is a submatrix of $\HH$.
    Then we know that $$|S| = \rank(K \cdot W) \leq \rank(\HH) \leq \min(\rank(\forw), \rank(\back)) = |S|,$$ so $\rank(H) = |S|$.
\end{proof}

\subsection{Proof of Proposition \ref{prop:prob-machine-off-by-transform}}\label{sec:appendix-proof-sim-trans}

We have two rank factorizations of $\HH$: $\forw \cdot \back = A \cdot V^T = \HH$.
Since all matrix factors involved are full-rank, we may take the Moore-Penrose inverse of $\forw$ and $\back$, which results in $\forw^\dagger A \cdot V^T \back^\dagger = I$.
Then $\forw^\dagger A$ is nonsingular and $V^T \back^\dagger$ is its inverse.

We take the product $(V^T \back^\dagger)$ to be $P$, so that $\forw^\dagger A = P^{-1}$. A consequence of the assumptions in Section \ref{sec:assumptions} is that $A_{\hists^{-ao, :}}$ is full-rank for all $a \in \mc A$ and $o \in \mc O$.
Thus, we could have repeated the argument above, but replacing $A$ with $A_{\hists^{-ao}}$ and $\HH$ with $\HH_{\hists^{-ao}, :}$, and find that $\forw_{\hists^{-ao}, :}^\dagger A_{\hists^{-ao}} = P^{-1}$, or equivalently $A_{\hists^{-ao}}^\dagger \forw_{\hists^{-ao}, :} = P$.

What remains is to show that we can apply $P$ to recover the POMDP initial belief, diagonal observation matrices, transition matrices, and final summing vector from the linear PSR models.
Let $a \in \mc A$ and $o \in \mc O$.
Following \cref{eqn:trans-obs-solve}, we know that $M^{ao} = A_{\hists^{-ao, :}}^\dagger \cdot \HH_{\hists^{ao}, :} \cdot {V^T}^{\dagger}$.
If we apply $P$ as a similarity transform, we find that
\begin{align*}
    P^{-1} M^{ao} P & = P^{-1} A_{\hists^{-ao, :}}^\dagger \cdot \HH_{\hists^{ao}, :} \cdot {V^T}^{\dagger} P                                \\
                    & = P^{-1} A_{\hists^{-ao, :}}^\dagger \cdot \forw_{\hists^{-ao},:} \cdot O^{ao} T^a \cdot \back \cdot {V^T}^{\dagger} P \\
                    & = P^{-1} P \cdot O^{ao} T^a P^{-1} P = O^{ao} T^a
\end{align*}

The initial belief vector $b_0$ and all ones vector $\mathbf 1$ can be recovered from the initial vectors $m_0$ and $m_\infty$ in the same manner (they only feature an inversion of $P$ on either the left or right sides, but not both as above).\hfill$\square$

\subsection{Proof of Lemma \ref{lem:unique-obs-unique-eigvals}}\label{sec:appendix-lemma-rand-sum}

The `only if' direction is immediate. We prove the `if' direction by proving its contrapositive.

Fix $i, j$ such that $1 \leq i < j \leq |\mc S|$. Suppose that there exists an $a \in \Afull$, $o \in \mc O$ such that $O^{ao}_{ii} \neq O^{ao}_{jj}$. Let $(ao)_1, \dots, (ao)_{|\Afull|\cdot|\mc O|}$ be an ordering on $\Afull \times \mc O$.
Let $a, b$ be two $|\Afull| \cdot |\mc O|$-dimensional vectors such that $a_k = O^{(ao)_{k}}_{ii}$ and $b_k = O^{(ao)_{k}}_{jj}$ for all $1 \leq k \leq |\Afull| \cdot |\mc O|$. Then we know that $a \neq b$.

Consider the event that $w \in \mathbb S^{|\Afull|\cdot|\mc O|}$, such that $\Lambda_{ii} = \sum_{k}^{|\Afull|\cdot|\mc O|} w_k O^{(ao)_k}_ii$  and $\Lambda_{jj} = \sum_{k}^{|\Afull|\cdot|\mc O|} w_k O^{(ao)_k}_jj$ are equivalent. Written in terms of the notation introduced above, we have that $\angleb{a - b, w} = 0$.
This means that $w$ must be contained in the hyperplane $H = \{x \in \R^{|\Afull|\cdot|\mc O|}: \angleb{x, a-b} = 0 \}$, which also passes through the origin.
We recognize that $H \cap \mb S^{|\Afull|\cdot|\mc O|}$ is a $|\Afull|\cdot|\mc O|-2$-dimensional submanifold (a lower-dimensional sphere) of $\mc S^{|\Afull|\cdot|\mc O|}$.
We know that the measure of this submanifold under the induced uniform measure on $\mb S^{|\Afull|\cdot|\mc O| - 1}$ from the Lebesgue measure on $\mathbb R^{|\Afull|\cdot|\mc O| - 1}$ is zero \cite{Lee2012}.
Thus, the probability of sampling $w$ so that $\Lambda_{ii} = \Lambda_{jj}$ is zero as well.
Therefore, the complement of this event, that $\Lambda_{ii} \neq \Lambda_{jj}$, must occur with probability one.
\hfill $\square$

\begin{remark}
    By a similar argument above, we obtain that $\diag(\Lambda) \neq \mathbf 0$ with probability 1, where $\mathbf 0$ is the zero vector.
    The argument replaces the discussion of the vector $a-b$, as constructed above, with the individual vectors $a$ or $b$.
\end{remark}

\subsection{Proof that Similarity Transform is Recovered up to Block-Diagonal Matrix}\label{sec:proof-recovery-up-to-block-diag}

First, we formalize the claim made in Sec. \ref{sec:trans-partitions}.

\begin{lemma}\label{lem:nonunique-obs-nonunique-P}
    Let $P'$ be the similarity transform as determined by an eigendecomposition of the random sum of \cref{eqn:rand-sim-obs-sum}.
    Without loss of generality, permute the columns of $P'$ and $P$ so that states in the same full-rank observability partition are in consecutive indices. Then

    \begin{equation}\label{eqn:block-similarity-transform}
        P^{-1}P' =
        \begin{pmatrix}
            Q_1 & 0   & \cdots & 0   \\
            0   & Q_2 & \cdots & 0   \\
            0   & 0   & \ddots & 0   \\
            0   & 0   &        & Q_k \\
        \end{pmatrix}
    \end{equation}
    where the blocks $Q_i \in \R^{|S_i| \times |S_i|}$ are nonsingular w.p. 1, where $|S_i|$ is the $i^{\text{th}}$ partition in the permuted index ordering.
\end{lemma}




Let $X$ denote the random sum as expressed in \cref{eqn:rand-sim-obs-sum}, and let $P\Lambda P^{-1}$, $P'\Lambda P'^{-1}$ be the two diagonalizations as discussed in Section \ref{sec:obs-as-eigenvalues}.
Let $S_{\Pi} = \{S_1, \dots, S_k\}$ be the full-rank observability partition.
By Lemma \ref{lem:unique-obs-unique-eigvals}, then $\Lambda_{ii} = \Lambda_{jj}$ for all $s^i, s^j$ in the same partition as $S \in S_{\Pi}$.
Furthermore, by the remark in Section \ref{sec:appendix-lemma-rand-sum}, we know that $\diag(\Lambda) \neq \mathbf 0$ with probability 1, where $\mathbf 0$ is the zero vector.

Suppose the indices of these matrices are ordered as stated in the hypothesis.
Since $X = P \Lambda P^{-1} = {P'} \Lambda P'^{-1}$, then we have that $P'^{-1} P \Lambda = \Lambda P'^{-1} P$ (e.g. $P'^{-1} P$ and $\Lambda$ commute).
Examining the entries of equation $P'^{-1} P \Lambda - \Lambda P'^{-1} P = 0$ yields that  $(\Lambda_{ii} - \Lambda_{jj})(P'^{-1} P)_{ij} = 0$.
If $s^i$ and $s^j$ are contained in separate partitions, then $\Lambda_{ii} - \Lambda_{jj} \neq 0$, so $(P'^{-1}P)_{ij} = 0$.
Thus, $P'^{-1} P$ has the desired block-diagonal structure.
Since we know that both $P$ and $P'$ are invertible, so must be $P'^{-1} P$, and the inverse of a nonsingular block-diagonal matrix is block-diagonal with the same block structure.
Thus, $P^{-1} P'$ has the structure claimed in \cref{eqn:block-similarity-transform}, and we know the blocks $Q_i$ are invertible as well. \hfill$\square$

\subsection{Proof of Theorem \ref{thm:coarse-grain-marg-and-trans}}\label{sec:appendix-proof-main-result}

Section \ref{sec:trans-partitions} claims that the matrix $P' R \diag(R^T P'^{-1} m_\infty)$ is a similarity transformation $\tilde P$ that satisfies the implication of Theorem \ref{thm:coarse-grain-marg-and-trans}.
As a reminder, $P'$ are the eigenvectors from the eigendecomposition of the matrix in \cref{eqn:rand-sim-obs-sum} and $R$ is a random block-diagonal rotation matrix with the same block structure as $P^{-1} P' = Q$ (Lemma \ref{lem:nonunique-obs-nonunique-P}), whose blocks are distributed over the Haar measure over the corresponding copy of $SO(n)$.

We begin our argument by first applying the similarity transformation $\tilde P$ to a learned PSR $m_0$, $\{M^{ao}: a \in \mc A, o \in \mc O\}$ and $m_\infty$.
We find that
\begin{align}
    m_0^T \tilde P                  & = b_\pi^T Q R \diag(R^T Q^{-1} \mathbf 1)\label{eqn:apply_tilde_P_m_0}                                            \\
    \tilde P^{-1} M^{ao} \tilde P & =  \diag(R^T Q^{-1} \mathbf 1)^{-1} R^T Q^{-1} \cdot \left(O^{ao} T^a\right)\cdot Q R \diag(R^T Q^{-1} \mathbf 1) \\
    \tilde P^{-1} m_\infty        & =  \diag(R^T Q^{-1} \mathbf 1)^{-1} R^T Q^{-1} \cdot\mathbf 1\label{eqn:apply_tilde_P_m_infty}
\end{align}
If we unpack the block structure of $Q$, $R$, and $\diag(R^T Q^{-1} \mathbf 1)$ in \cref{eqn:apply_tilde_P_m_0,eqn:apply_tilde_P_m_infty}, we find
\begin{align}
    [m_0^T \tilde P]_{S_i}         & = [b_\pi^T]_{S_i} Q_i R_i \diag(R_i^T Q_i^{-1} [\mathbf 1]_{S_i})\label{eqn:apply_tilde_P_m_0_blocks}                         \\
    [\tilde P^{-1} m_\infty]_{S_i} & = \diag(R_i^T Q_i^{-1} [\mathbf 1]_{S_i})^{-1} R_i^T Q_i^{-1} \cdot [\mathbf 1]_{S_i}\label{eqn:apply_tilde_P_m_infty_blocks}
\end{align}
where $R_i$ and $Q_i$ are the blocks associated with the full-rank observability partition $S_i$, and $[U]_{\mc I}$ represents the values of the vector or matrix-valued quantity $U$ indexed by a set of indices $\mc I$.

First, we must justify that the relations expressed in \crefrange{eqn:apply_tilde_P_m_0}{eqn:apply_tilde_P_m_infty_blocks} are well-defined.
We already know $R$ and $Q$ are nonsingular.
We must then show all entries of the vector $R^T Q^{-1} \mathbf 1$ are nonzero to allow for the existence of $\diag(R^TQ^{-1}\mathbf 1)$.
This fact is a consequence of known properties of the Haar measure over special orthogonal matrices \cite[Section 1.2]{meckesRandomMatrix2019}.
Fix a full-rank observability partition $S_i$.
It is known that corresponding rotation matrix blocks $R_i^T$ and $R_i$ are identically distributed with respect to the Haar measure on $SO(|S_i|)$ \cite[pg. 18]{meckesRandomMatrix2019}.
Furthermore, since $Q_i$ is nonsingular, we know that $Q_i^{-1} \mathbf 1$ is not the zero vector.
Thus, it is also known that the random vector $R_i^T \cdot (Q_i^{-1} [\mathbf 1]_{S_i})$ is uniformly distributed over the $(|S_i|-1)$-sphere with radius $\norm{Q_i^{-1} [\mathbf 1]_{S_i}}_2$ \cite[pg. 19-20, 26]{meckesRandomMatrix2019}. By the same argument discussed in the proof of Lemma \ref{lem:unique-obs-unique-eigvals}, the entries of $R_i^T Q_i^{-1}[\mathbf 1]_{S_i}$ must be nonzero with probability one.
By taking a union bound over all full-rank observability partitions, \textit{all} entries of $R^T Q^{-1} \mathbf 1$ must be nonzero with probability one as well.

What remains is to prove the correctness of the relations \crefrange{eqn:sum_psr_vec_parts}{eqn:obtain_sum_vector} in Theorem~\ref{thm:coarse-grain-marg-and-trans}.
The expression that we obtain $\tilde P^{-1} m_\infty = \mathbf 1$ is immediate from \cref{eqn:apply_tilde_P_m_infty,eqn:apply_tilde_P_m_infty_blocks}.
First, we justify \cref{eqn:sum_psr_vec_parts}.
Fix a full-rank observability partition $S_i$.
Then
\begin{align*}
    \sum_{s^j \in S_i} [\tilde b_\pi]_j
     & = [m_0^T \tilde P]_{S_i} \cdot [\mathbf 1]_{S_i}                                                                                                         \\
     & = [b_\pi^T]_{S_i} Q_i R_i \diag(R_i^T Q_i^{-1} [\mathbf 1]_{S_i}) \cdot \diag(R_i^T Q_i^{-1} [\mathbf 1]_{S_i})^{-1} R_i^T Q_i^{-1} \cdot [\mathbf 1]_{S_i} \\
     & = [b_\pi^T]_{S_i} \cdot [\mathbf 1]_{S_i}                                                                                                                   \\
     & = \sum_{s^j \in S_i} [b_\pi]_j
\end{align*}
The second line applies \cref{eqn:apply_tilde_P_m_0_blocks,eqn:apply_tilde_P_m_infty_blocks}.
The proof for \cref{eqn:sum_apply_paz} is nearly the same, and can be reached by deriving an analagous expression to \cref{eqn:apply_tilde_P_m_0_blocks} by first multiplying out the corresponding sequence of matrices $\tilde P^{-1} M^{ao} \tilde P$ and unpacking the block structure of $Q$ and $R$ again.\hfill$\square$

\subsection{Proof of Full-Rank Transition Claim}\label{sec:appendix-full-rank-trans}

We formally state the claim made in the deliberation of Section \ref{sec:main-result}

\begin{proposition}\label{prop:full-rank-actions}
    Let $T$ be an $n\times n$ matrix, with rows that are all zeros except for a single entry of 1 per row. Let $p \in [0, 1)$, and $p \neq 1/2$.
    Then the convex combination $p T + (1 - p) I$ is nonsingular.
\end{proposition}
\begin{proof}
    We first observe that the proof is immediate if $p = 0$, so we focus on the case for $p \in (0,1)$.
    Suppose, for the sake of contradition, that there exists $v \in \R^n$ such that matrix-vector product $\left(pT + (1-p)I\right) v = 0$. This must be true if and only if
    $$
        Tv = \left(\frac{p - 1}{p}\right)v,
    $$
    or that $v$ is an eigevector of $T$ with eigenvalue $(p-1)/p$.

    We claim that the eigenvalues of $T$ are either zero or roots of unity. If this claim is true, we arrive at a contradiction, because if $p \neq 1/2$ and $p \in (0, 1)$, then $(p-1) / p$ cannot be equal $-1$.

    We prove this claim by induction on the number of rows and columns.
    As the base-case, we take a $1 \times 1$ ``matrix" $\begin{pmatrix} 1 \end{pmatrix}$.
    The eigenvalue of this matrix is unity.
    Next, we assume that the claim holds for $m \times m$ matrices with rows of all zeros except for a single one.
    Suppose we have a $(m + 1) \times (m + 1)$ matrix $T'$ of the same structure.
    If $T'$ is a permutation matrix, then we know its eigenvalues are roots of unity \cite{dingWhenMatrix2014},
    so suppose that $T'$ is not a permutation matrix.
    Then $T'$ must have at least one columns that is all zeros.
    We then examine the characteristic polynomial $\phi(T')$.
    Without loss of generality, assume that column is the first column of the matrix. Then, we can write out the expression for the characteristic polynomial~$\phi(T')$:
    \begin{align*}
        \phi(T')
         & = \det\left(T' - \lambda I\right)              \\
         & = \det\left(
        \begin{array}{c | c}
                -\lambda & \cdots                  \\
                \hline
                0        & T'_{2:, 2:} - \lambda I \\
            \end{array}\right)                \\
         & = -\lambda \cdot \det(T'_{2:, 2:} - \lambda I) \\
         & = -\lambda \cdot \phi(T'_{2:,2:})
    \end{align*}
    where $T'_{2:, 2:}$ is the $m \times m$ submatrix of $T'$ that omits the first row and column of $T'$. This submatrix is also a matrix with all zeros for every row except for a single one, since we eliminated a column of only zeros from $T'$.
    Thus, we know that the eigenvalues of $T'$ are $0$ and the eigenvalues of $T'_{2:, 2:}$, which, by the induction hypothesis, are also zero and roots of unity.\hfill$\square$
\end{proof}

\subsection{Proof of Theorem \ref{thm:same_hankel}}\label{sec:proof_unid}

We first begin by characterizing the stochastic process $\{s_t\}_{t=0}^{\infty}$, the marginals of states variables during exploration ($a_t \sim \Unif(\mc A)$ i.i.d.). We begin by proving the following lemma:

\begin{lemma}
    The stochastic process $\{s_t\}_{t\geq0}$ is a Markov chain, and its associated transition matrix is $$T = \frac 1 {|\mc A|} \sum_{a \in \mc A} T^a.$$
\end{lemma}
\begin{proof}
    We first begin with conditional joint likelihood that we would like to factorize and unmarginalize present and past actions.
    Because actions are sampled i.i.d., we have that:
    \begin{align*}
         & \PP(s_t, s_{t+2} | s_{t+1}) \\
         & = \sum_{a_t, a_{t+1} \in \mc A} \PP(s_t, s_{t+2} | a_t, s_{t+1}, a_{t+1}) \PP(a_t | s_{t+1}) \PP(a_{t+1} | s_{t+1})              \\
         & = \sum_{a_t, a_{t+1} \in \mc A} \PP(s_t | a_t, s_{t+1}) \PP(s_{t+2} | s_{t+1}, a_{t+1})\PP(a_t | s_{t+1}) \PP(a_{t+1} | s_{t+1}) \\
         & = \sum_{a_t, a_{t+1} \in \mc A} \PP(s_t, a_t| s_{t+1}) \PP(s_{t+2}, a_{t+1}| s_{t+1})                                            \\
         & = \PP(s_t | s_{t+1}) \PP(s_{t+2} | s_{t+1})                                                                                      \\
    \end{align*}
    The third line applies the fact that $s_{t+2} \perp s_t | s_{t+1}, a_t, a_{t+1}$  and $s_{i} \perp a_{i+1} | s_{i+1}$ due to the conditional independencies induced by a POMDP.

    We then observe that the transition likelihood between two states is the averaged transition likelihood across all actions.
    \begin{equation}\label{eqn:s_chain_trans}
        \PP(s_{t+1} | s_t) = \sum_{a_t \in \mc A}\PP(s_{t+1} | s_t, a_t) \PP(a_t | s_t) = \sum_{a_t \in \mc A} \frac{\PP(s_{t+1} | s_t)}{|\mc A|}
    \end{equation}
    The last equality invokes that $a_t \sim \Unif$ i.i.d. Thus, we can conclude from \cref{eqn:s_chain_trans} that the transition matrix of this Markov chain $T$ can be written as the average of the transition matrices of the POMDP, e.g. $T = \sum_{a \in \mc A} T^a / |\mc A|$.
    \hfill $\square$
\end{proof}
We also observe that the stationary distribution of this Markov chain is $b_\pi$, as defined in the proof of Lemma \ref{lem:assumptions-consequences}, Appendix \ref{sec:proof-assumption-consequences} and is unique (otherwise, we would violate ergodicity of the Markov chain $\{s_t, a_t, o_t\}$).

Now, we are ready to prove the main claim of Theorem \ref{thm:same_hankel}.
Let the transition and diagonal observation matrices of the original POMDP be $\mc T = \{T^a : a \in \mc A\}$ and $\mc Z = \{O^{ao} : a \in \mc A, o \in \mc O\}$, and suppose that $\mc T': \{{T^a}' : a \in \mc A\}$ and $\mc Z' : \{{O^{ao}}': a \in \mc A, o \in \mc O\}$ are the transition and diagonal matrices of the POMDP after transformation $P$ is applied.
Since $b_\pi$ is the unique stationary distribution of Markov chain $\{s_t\}_{t \geq 0}$ evolving according to POMDP with transition matrices $\mc T$, then  $b_\pi$ must be a unique left-eigenvector of matrix $T = \sum_{a \in \mc A} T^a / |\mc A|$ with eigenvalue 1 \cite{norrisDiscretetimeMarkov1997}.
Therefore, the row-stochastic vector $b_\pi'^T = b_\pi^T P^{-1}$ must be a left-eigenvector of the transition matrix $T' = \sum_{a \in \mc A} T'^a / |\mc A|$ because

\begin{align*}
     & b_\pi'^T \left(\frac 1 {|\mc A|} \sum_{a \in \mc A} T'^a \right)               \\
     & = b_\pi^T P^{-1} \left(\frac{1} {|\mc A|} \sum_{a \in \mc A} P T^a P^{-1}\right) \\
     & = b_\pi^T \left(\frac{1} {|\mc A|} \sum_{a \in \mc A}  T^a \right) P^{-1}        \\
     & = b_{\pi}^T P^{-1}                                                               \\
     & = b_{\pi}'^T
\end{align*}
From the relation above, we conclude that $b_\pi'$ must be a unique stationary distribution for the stochastic process $\{s'_{t}\}_{t \geq 0}$ that evolves according to POMDP transition matrices $\mc T'$.

Lastly, fix $\hist = (a_1, o_1, \dots, a_t, o_t)$ and $\ttest = (a_t, o_t, \dots, a_n, o_n)$.
We evaluate \crefrange{eqn:forward-derivation}{eqn:hankel-entry} to relate Hankel matrices $\HH$ and $\HH'$:
\begin{align*}
    \HH_{\hist, \ttest} & = b_\pi^T O^{a_1 o_1} T^{a_1} O^{a_2 o_2} T^{a_2} \cdots O^{a_n o_n} T^{a_n} \mathbf 1                                           \\
                        & = b_\pi^T P^{-1} \cdot P O^{a_1 o_1} P^{-1} \cdot P T^{a_1} P^{-1} \cdots P O^{a_n o_n} P^{-1} \cdot P T^{a_n} P^{-1} \cdot P \mathbf 1 \\
                        & = b_\pi'^T {O^{a_1o_1}}' {T^{a_1}}' {O^{a_2o_2}}' {T^{a_2}}' \cdots {O^{a_no_n}}' {T^{a_n}}' \mathbf 1                          \\
                        & = \HH'_{\hist, \ttest}.
\end{align*}
The third line involes the fact that the provided similarity transform in the perturbed sense-float-reset counterexample (\cref{eqn:counterexample}) has rows that sum to one.
Thus, we conclude that $\HH = \HH'$. \hfill$\square$

\section{Appendix: Additional Algorithmic Details}\label{sec:appendix-algorithm-details}
\subsection{Parameters Introduced for Finite Data}\label{sec:finite-data-params}


Our derivations so far have assumed to be in the asymptotic regime where we have made perfect estimates of the Hankel matrix.
In practice, with finite data, we only have the empirical Hankel matrix, $\HHhat$, which is subject to random perturbations.
Naturally, some adjustments to the calculations expressed in the previous section must be made to account for estimation error.
There are three operations where estimation error influences the learning procedure: rank estimation via truncated SVD, determining full rank transition matrices $M^a$, and obtaining the observability partition to compute the random matrix $R$ in Sec. \ref{sec:trans-partitions}.
For the Hankel matrix rank, we find that introducing a threshold on the low-rank approximation's reciprocal condition number to be sufficient.
To test for transition full-rankness, we found that a minimum singular value $\sigma_{\text{min}}$ threshold was acceptable.
To find the full-rank observability partition, we consider two observation distributions to be equivalent if their $L^1$ norm falls below a threshold~$\tau_{\mt{obs}}$.

While the transition and observation likelihoods computed from the data will converge to the values true values asymptotically, approximation error prevents us from directly reading the parameter estimates as probabilities \cite{guoPACRL2016,azizzadenesheliReinforcementLearning2016}.
Before using the learned model we project all parameters back to the probability simplex by minimizing the $L_2$ norm by quadratic programming.

Algorithm pseudocode can be found in \cref{alg:learn-pomdp}.
We note that for all experiments, while it is theoretically correct to construct a block-diagonal rotation matrix $R$ by the full-rank observability partition as stated above, we find in practice it is sufficient to multiply by fully-dense random rotation matrix $R'$ after computing the initial SVD (line \ref{ln:svd}).
This modification uses $A R'$ and $R'^T V$ to compute the linear PSR and takes $R = I$ instead at line \ref{ln:block-rand-rot}.
We still require a $\tau_{obs}$ parameter to compute the partition-level transition errors in \cref{fig:all-plots,fig:swap-reward}.

\begin{algorithm}[t]
\caption{Learn-POMDP}\label{alg:learn-pomdp}
\begin{algorithmic}[1]
    \Require Dataset $\mc D = (a_1o_1, a_2o_2\dots)$, reciprocal cond. number $1 / \kappa$, substring length $L$, minimum trans. mat singular value $\sigma_{min}$, observation sim. threshold $\tau_{obs}$:
    \State $\mathit{substrings} \gets \{(a_io_ia_{i+1}o_{i+1}\dots a_{i+k}o_{i+k})\}_{i=1}^{|\mc D|/2 - L}$
    \State $\HH \gets \tsc{EstimateHankel}(\mathit{substrings})$ \label{lp-ln:estimate-hankel} \Comment{Entries estimated via Eq. \ref{eqn:emp-hankel-est}.}
    \State $U, \Sigma, V^T \gets \tsc{TruncatedSVD}(\HH, r, 1 / \kappa)$
    \State $A \gets U \Sigma$ \label{ln:svd}
    \State $m_0, \{M^{ao}: a \in \mc A, o \in \mc O\},  m_\infty \gets \tsc{ComputePSR}(A, V^T, \HH)$ \Comment{via Eqs. \ref{eqn:trans-obs-solve}-\ref{eqn:final-vec-solve}.}
    \State $M_{obs} = []$
    \For{$a \in \mc A$}
        \State $M^a \gets \sum_{o \in \mc O} M^{ao}$
        \If{$\tsc{MinSingularValue}(M^a) > \sigma_{min}$} \Comment{Detect full-rank actions.}
            \For{$o \in \mc O$}
                \State $(M_{obs})$.append($M^{ao}(M^a)^{-1}$)
            \EndFor
        \EndIf
    \EndFor
    \State $w_1, \dots w_{|M_{obs}|} \sim \Unif(\mb S^{|M_{obs}| - 1})$
    \State $P' \gets \tsc{Eigenvectors}(\sum_{i=1}^{|M_{obs}|} w_i (M_{obs})_i)$ \Comment{via Eq. \ref{eqn:rand-sim-obs-sum}. Eigenvectors form columns of $P'$.}
    \For{$M^{ao} (M^a)^{-1} \in M_{obs}$}
        \State $O^{ao} \gets P'^{-1} M^{ao} (M^a)^{-1} P'$
    \EndFor
    \State $[S_1, \dots, S_k] \gets \tsc{DetectPartitions}(\{O^{ao}\}, \tau_{obs})$  \Comment{via procedure in Sec. \ref{sec:finite-data-params}.}
    \For{$S_i \in [S_1, \dots, S_k]$}
        \State $R_i \sim \Unif(SO(|S_i|))$
    \EndFor
    \State $R \gets \tsc{BlockDiag}([R_1, \dots, R_k], [S_1, \dots, S_k])$ \label{ln:block-rand-rot}
    \State $\tilde P \gets P' R \diag(R^T P'^{-1}m_\infty)$  \Comment{Blocks are specified by indices in partitions $S_1, \dots, S_k$.}
    \State $\tilde b^T \gets m_0^T \tilde P$
    \For{$(a, o) \in \mc A \times \mc O$}
        \State $\tilde O^{ao} \tilde T^{a} \gets \tilde P^{-1} M^{ao} \tilde P$
    \EndFor
    \Ensure $\tilde b$, $\{\tilde O^{ao} \tilde T^a: a \in \mc A, o \in \mc O\}$ 
\end{algorithmic}
\end{algorithm}

\subsection{Runtime Complexity in Floating-Point Operations}\label{sec:runtime}
\newcommand{\aonobs}{(|\mc A||\mc O|)^{\nobs+1}}
The runtime of our approach, which we measure in floating-point operations, is dominated by the rank factorization of the Hankel matrix and computation of the PSR update matrices.
We define the \textit{full-observability length} of a POMDP (and notate as $\nobs$) to be the smallest length of histories and tests so that the Hankel matrix, whose rows and columns are indexed by action-observation sequences enumerated up to this length, is full-rank.
Suppose we are given a Hankel matrix that enumerates histories up to $\nobs + 1$ in the rows (so that $A_{\hists^{-ao}, :}$ in \cref{eqn:trans-obs-solve} is full-rank) and tests up to length $\nobs$ in the columns.
The size of the Hankel matrix, then, must be $O(\aonobs) \times O(\aonobs)$.
Computing the truncated SVD with an appropriately set singular value, threshold, then, has runtime $O(|\mc S|\cdot\aonnobs)$.
To compute the PSR update matrices, we pseudoinvert the right rank factor once, which also has complexity $O(|S| \cdot \aonnobs)$.
Then, to compute each $M^{ao}$, we must pseudoinvert the right factor $A_{\hists^{-ao}, :}$, which has runtime $O(|\mc S| \cdot \aominnobs)$, and then compute the product $A_{\hists^{-ao}, :}^{\dagger} \HH_{\hists^{ao}, :} (V^T)^{\dagger}$, which has runtime $$O\left(|\mc S| \cdot \aonnobsmin\right) + O\left(|\mc S|^2 \aonobs\right).$$
Putting everything together, we have a full runtime of:
\begin{equation}\label{eqn:runtime}
    O\left(|\mc S|  \aonnobs + |\mc S|^2 \aopnobs\right)
\end{equation}
Interestingly, our calculation suggests a runtime that is polynomial when the observability length $\nobs$ scales favorably when as the number of states of a particular class of POMDPs increases, which aligns with prior work on Hidden Markov Models \cite{huangMinimalRealization2016}.
Further investigation of computational tractability learning framework (e.g. PAC-learning) for the multi-action case would be an interesting direction of future work.

\section{Appendix: Additional Experimental Details}\label{sec:appendix-exps}

\subsection{Algorithm Parameter Selection}\label{sec:appendix-alg-params}

As discussed in Appendix \ref{sec:appendix-algorithm-details}, the behavior performance of our learning algorithm depends on a few manually-specified parameters.
This section reviews all the parameters that must be specified to run our approach, and the parameters values selected for our experiments (for a summary, see \cref{tab:exp-params}).

For Hankel estimation, of practical concern is the selection of the size Hankel matrix to estimate, or the sequences to include as row and column indices.
Like many other approaches \cite{hsuSpectralAlgorithm2012,balleSpectralLearning2014}, we use every possible action-observation sequence up to a certain length.
We expose this length as an algorithm parameter.
While smaller lengths will result in faster convergence of matrix entry estimates, selecting a length that is too short may result in a Hankel matrix whose approximate rank is strictly less than the number of states of the system.
Our chosen lengths are included in \cref{tab:exp-params}.
Automatically determining the proper Hankel size is an open question since the development of spectral approaches for PSRs \cite{wolfeLearningPredictive2005,bootsClosingLearningplanning2011,balleSpectralLearning2014}, and remains an interesting question for future work.

The main parameter associated with learning PSRs is centered around the number of singular components to be used when computing the rank factorization of the Hankel matrix (Section \ref{sec:psrs}).
As mentioned in Section \ref{sec:finite-data-params}, the main way to do this is by specifying a lower threshold on the empirical Hankel matrix lower rank approximation's reciprocal condition number.
Empirically, we observe that empirical Hankel matrices tend to become more singular as the amount of data used to estimate them increases (\cref{fig:rcond-over-time}).
While any sufficiently small positive threshold may work with large amounts of data, in practice, larger thresholds will more quickly identify the number of states, at risk of omitting states.
Of practical concern is the maximum number of singular values to compute to avoid computing an SVD of the \textit{entire} Hankel matrix.
The specified values for our experiments are shown in \cref{tab:exp-params} under $1 / \kappa$ and `No. SVD,' respectively.
\begin{figure}[t]
    \centering
    \includegraphics[width=\linewidth]{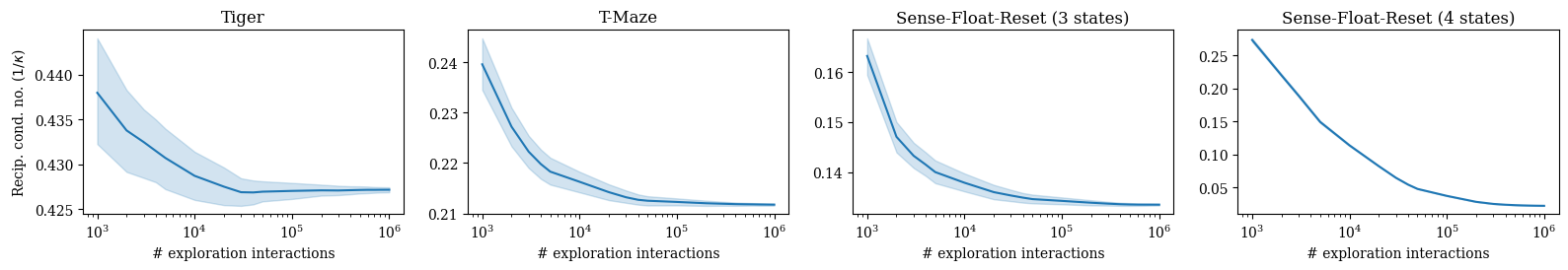}
    \caption{The ratio of the $r$th condition number over the largest condition number of Hankel matrices as the amount of observed data increases, where $r = |S|$ is the number of states of the POMDP. Each plot is averaged over 100 runs. The matrices become \textit{more} singular as the amount of data increases. The sizes of the Hankel matrix correspond with \cref{tab:exp-params}.}\label{fig:rcond-over-time}
\end{figure}

When recovering the observation distributions and partition-level transitions, we must specify thresholds to determine transition matrix full-rankness and a threshold that determines when observation distributions are similar (Section \ref{sec:finite-data-params}).
Inverting a near-singular matrix is highly undesirable when computing the matrices to joint-diagonalize in \cref{eqn:diag-matrix-prep}, so we specify conservatively high threshold on the smallest singular value on the smallest singular value $\sigma_{min}$ of the transition matrix.
As discussed in Appendix \ref{sec:finite-data-params}, in practice, a threshold $\tau_{obs}$ is not required to compute a random block-diagonal rotation matrix $R$. However, $\tau_{obs}$ is still required to compute partition-level transition likelihoods errors reported in \cref{fig:all-plots,fig:swap-reward}.
We set a conservativitely high threshold $\tau_{obs}$ to merge observation distributions aggressively when plotting those figures.


All PO-UCT planners require specification of a upper-confidence bound (UCB) constant to balance exploiting current estimated action value and exploring new actions.
For all experiments, we use a UCB constant of $c = 2$.
Our planners also limits all searches to a depth of three, and performs a fixed 1000 simulations per planning step.

\begin{table}
    \caption{Chosen parameters for each domain in planning experiments described in Section \ref{sec:experiments}, Fig. \ref{fig:all-plots}. Up and down arrows indicate a upper or lower threshold, respectively. The tuple reported for maximum indexing sequence lengths is ordered: (rows, columns).}\label{tab:exp-params}
    \centering
    \begin{tabular}{lllllll|}
        \toprule 
                                     & $1/\kappa\downarrow$ & $\sigma_{min}\uparrow$ & $\HH$ max. seq. len. $\uparrow$ & No. SVD $\uparrow$ & $\tau_{obs}\uparrow$\\ \midrule
        Tiger                        & 0.34                   & 0.1                    & (2, 1)                    & 20                 & 0.1\\ 
        T-Maze                       &     0.1                   &         0.01               & (2, 1)                    & 20                 & 0.1\\ 
        Sense-Float-Reset (3 states) & 0.1                    &                0.1        & (3, 2)                    & 20                 & 0.1\\ 
        Sense-Float-Reset (4 states) &       0.015                 &                   0.1     & (4, 3)                    & 20                & 0.5\\ 
        \bottomrule
    \end{tabular}
\end{table}

\subsection{Sensitivity Analysis and Wall-Clock Runtime Estimates}

Additionally, we have included an analysis on the sensitivity of the truncated SVD step to both Hankel size and rank tolerance $1 / \kappa$ (see Appendix \ref{sec:appendix-alg-params} for a description on parameters).
As experiments results in \cref{fig:all-plots} suggest, convergence of the number of states is a key step to the convergence of the overall algorithm.

\Cref{tab:tmaze-sa-states} reports the number of estimated states against a variety of Hankel sizes and rank tolerances (and the rank of the Hankel matrix in limit of infinite data).
Estimates of observation and transition likelihood error can be found in \cref{tab:tmaze-sa-obs-err,tab:tmaze-sa-trans-err} respectively.
Runtimes can be found in \cref{tab:tmaze-sa-runtimes}.
Maximum Hankel size was determined to be the largest to allow for a RAM allocation under 64Gb on eight cores allocated on an Intel Xeon Gold 6140 CPU on a shared cluster.
We observe that while more aggressive rank thresholds may arrive at the correct estimate with less data, they may also lead to an underestimation of the number of states.
Lower thresholds will more likely underestimate the number of states, and require more data before the correct estimate is reached.
Furthermore, larger Hankel sizes are required to estimate POMDPs with larger states, which tend to have longer observability lengths (Appendix \ref{sec:runtime}).
For example, a Hankel size that enumerates  histories of length four and tests of length three cannot fully capture a 14-state T-Maze.
As Hankel size increases, so do the acceptable thresholds to estimate the number of states.
Transition and observation likelihood errors, however, appear to be less influenced by Hankel size.
These results suggest that the largest possible Hankel accomodated by running time and memory should be used for ease of selection of the remaining algorithm parameters.

\begin{table}
    \caption{Sensitivity analysis on the \textbf{estimated rank} of the Hankel matrix based chosen Hankel sizes and rank tolerances estimated from $10^7$ interactions in T-Maze environments with varying numbers of states. Hankel size is represented in the maximum lengths of action-observation sequences used to index the row and columns of the Hankel matrix, respectively. Rank tolerance is specified as $1 / \kappa$, as discussed in Appendix \ref{sec:appendix-alg-params}. Results are reported up to two significant figures, with trailing zeros truncated for space. Hankel rank for the corresponding Hankel size in the limit of infinite data is included in the GT column. Results reported are mean and standard deviation of the number of estimated states, aggregated over 20 seeds. A value of `nan' is reported when no full-rank actions were found.}\label{tab:tmaze-sa-states}
    \centering
    \begin{tabular}{ll@{\hspace{1.50em}}r@{\hspace{1.50em}}llllll}
        \toprule
                               & $1 / \kappa$           & GT & 1e-01    & 1e-02       & 1e-03        & 1e-04       & 1e-05       & 1e-06     \\
        $n$ states             & $\mathcal H$ seq. len. &    &          &             &              &             &             &           \\
        \midrule
        \multirow[t]{3}{*}{4}  & (2, 1)                 & 4  & $4\pm 0$ & $4\pm 0$    & $6\pm .74$   & $13\pm .73$ & $14\pm .46$ & $14\pm 0$ \\
                               & (3, 2)                 & 4  & $4\pm 0$ & $4\pm 0$    & $20\pm 0$    & $20\pm 0$   & $20\pm 0$   & $20\pm 0$ \\
                               & (4, 3)                 & 4  & $4\pm 0$ & $4\pm 0$    & $20\pm 0$    & $20\pm 0$   & $20\pm 0$   & $20\pm 0$ \\
        \cline{1-9}
        \multirow[t]{3}{*}{6}  & (2, 1)                 & 5  & $4\pm 0$ & $5\pm 0$    & $7.9\pm .94$ & $16\pm .73$ & $18\pm .4$  & $18\pm 0$ \\
                               & (3, 2)                 & 6  & $5\pm 0$ & $6\pm 0$    & $20\pm 0$    & $20\pm 0$   & $20\pm 0$   & $20\pm 0$ \\
                               & (4, 3)                 & 6  & $6\pm 0$ & $6\pm 0$    & $20\pm 0$    & $20\pm 0$   & $20\pm 0$   & $20\pm 0$ \\
        \cline{1-9}
        \multirow[t]{3}{*}{8}  & (2, 1)                 & 6  & $5\pm 0$ & $6\pm 0$    & $10\pm 1.1$  & $20\pm .3$  & $20\pm 0$   & $20\pm 0$ \\
                               & (3, 2)                 & 8  & $6\pm 0$ & $8\pm 0$    & $20\pm 0$    & $20\pm 0$   & $20\pm 0$   & $20\pm 0$ \\
                               & (4, 3)                 & 8  & $6\pm 0$ & $8\pm 0$    & $20\pm 0$    & $20\pm 0$   & $20\pm 0$   & $20\pm 0$ \\
        \cline{1-9}
        \multirow[t]{3}{*}{10} & (2, 1)                 & 7  & $6\pm 0$ & $7\pm 0$    & $13\pm 1.4$  & $20\pm 0$   & $20\pm 0$   & $20\pm 0$ \\
                               & (3, 2)                 & 9  & $6\pm 0$ & $8\pm 0$    & $20\pm 0$    & $20\pm 0$   & $20\pm 0$   & $20\pm 0$ \\
                               & (4, 3)                 & 10 & $7\pm 0$ & $10\pm .3$  & $20\pm 0$    & $20\pm 0$   & $20\pm 0$   & $20\pm 0$ \\
        \cline{1-9}
        \multirow[t]{3}{*}{12} & (2, 1)                 & 8  & $7\pm 0$ & $8\pm 0$    & $17\pm .93$  & $20\pm 0$   & $20\pm 0$   & $20\pm 0$ \\
                               & (3, 2)                 & 10 & $7\pm 0$ & $9\pm 0$    & $20\pm 0$    & $20\pm 0$   & $20\pm 0$   & $20\pm 0$ \\
                               & (4, 3)                 & 12 & $8\pm 0$ & $14\pm 1.2$ & $20\pm 0$    & $20\pm 0$   & $20\pm 0$   & $20\pm 0$ \\
        \cline{1-9}
        \multirow[t]{3}{*}{14} & (2, 1)                 & 9  & $8\pm 0$ & $9\pm 0$    & $20\pm .57$  & $20\pm 0$   & $20\pm 0$   & $20\pm 0$ \\
                               & (3, 2)                 & 11 & $8\pm 0$ & $10\pm 0$   & $20\pm 0$    & $20\pm 0$   & $20\pm 0$   & $20\pm 0$ \\
                               & (4, 3)                 & 13 & $9\pm 0$ & $20\pm 0$   & $20\pm 0$    & $20\pm 0$   & $20\pm 0$   & $20\pm 0$ \\
        \cline{1-9}
        \bottomrule
    \end{tabular}
\end{table}

\begin{table}
    \caption{Sensitivity analysis on the \textbf{observation error} (in $L_1$ norm) associated with Table \ref{tab:tmaze-sa-states}. Estimates are only taken when the number of estimated states is equivalent to the ground truth POMDP. For all experiments the full-rank transition threshold $\sigma_{min}$ is set to $0.01$.  A value of `nan' is reported when no full-rank actions were found. If no standard deviation is included, only one seed of twenty succeeded in finding a full-rank action.}\label{tab:tmaze-sa-obs-err}
    \centering
    \begin{tabular}{llllllll}
        \toprule
                               & $1 / \kappa$           & 1e-01             & 1e-02             & 1e-03        & 1e-04 & 1e-05 & 1e-06 \\
        $n$ states             & $\mathcal H$ seq. len. &                   &                   &              &       &       &       \\
        \midrule
        \multirow[t]{3}{*}{4}  & (2, 1)                 & 0.09 $\pm$ 0.2    & 0.11 $\pm$ 0.28   & nan          & nan   & nan   & nan   \\
                               & (3, 2)                 & 0.031 $\pm$ 0.012 & 0.033 $\pm$ 0.015 & nan          & nan   & nan   & nan   \\
                               & (4, 3)                 & 0.026 $\pm$ 0.011 & 0.025 $\pm$ 0.012 & nan          & nan   & nan   & nan   \\
        \cline{1-8}
        \multirow[t]{3}{*}{6}  & (2, 1)                 & nan               & nan               & 0.47  & nan   & nan   & nan   \\
                               & (3, 2)                 & nan               & 0.2 $\pm$ 0.45    & nan          & nan   & nan   & nan   \\
                               & (4, 3)                 & 0.063 $\pm$ 0.025 & 0.077 $\pm$ 0.054 & nan          & nan   & nan   & nan   \\
        \cline{1-8}
        \multirow[t]{3}{*}{8}  & (2, 1)                 & nan               & nan               & nan          & nan   & nan   & nan   \\
                               & (3, 2)                 & nan               & 0.38 $\pm$ 0.24   & nan          & nan   & nan   & nan   \\
                               & (4, 3)                 & nan               & 0.59 $\pm$ 1.5    & nan          & nan   & nan   & nan   \\
        \cline{1-8}
        \multirow[t]{3}{*}{10} & (2, 1)                 & nan               & nan               & nan          & nan   & nan   & nan   \\
                               & (3, 2)                 & nan               & nan               & nan          & nan   & nan   & nan   \\
                               & (4, 3)                 & nan               & 0.29 $\pm$ 0.13   & nan          & nan   & nan   & nan   \\
        \cline{1-8}
        \multirow[t]{3}{*}{12} & (2, 1)                 & nan               & nan               & nan          & nan   & nan   & nan   \\
                               & (3, 2)                 & nan               & nan               & nan          & nan   & nan   & nan   \\
                               & (4, 3)                 & nan               & 1.7 $\pm$ 0.017   & nan          & nan   & nan   & nan   \\
        \cline{1-8}
        \multirow[t]{3}{*}{14} & (2, 1)                 & nan               & nan               & nan          & nan   & nan   & nan   \\
                               & (3, 2)                 & nan               & nan               & nan          & nan   & nan   & nan   \\
                               & (4, 3)                 & nan               & nan               & nan          & nan   & nan   & nan   \\
        \cline{1-8}
        \bottomrule
    \end{tabular}
\end{table}

\begin{table}
    \caption{Sensitivity analysis on the \textbf{transition error} (in $L_1$ norm) associated with Table \ref{tab:tmaze-sa-states}. Values are reported using the same estimation protocol as Table \ref{tab:tmaze-sa-obs-err}, except transition likelihoods were measured.}\label{tab:tmaze-sa-trans-err}
    \centering

    \begin{tabular}{llllllll}
        \toprule
                               & $1 / \kappa$           & 1e-01             & 1e-02              & 1e-03        & 1e-04 & 1e-05 & 1e-06 \\
        $n$ states             & $\mathcal H$ seq. len. &                   &                    &              &       &       &       \\
        \midrule
        \multirow[t]{3}{*}{4}  & (2, 1)                 & 0.095 $\pm$ 0.32  & 0.067 $\pm$ 0.21   & nan          & nan   & nan   & nan   \\
                               & (3, 2)                 & 0.018 $\pm$ 0.015 & 0.028 $\pm$ 0.045  & nan          & nan   & nan   & nan   \\
                               & (4, 3)                 & 0.017 $\pm$ 0.022 & 0.011 $\pm$ 0.0077 & nan          & nan   & nan   & nan   \\
        \cline{1-8}
        \multirow[t]{3}{*}{6}  & (2, 1)                 & nan               & nan                & 0.46  & nan   & nan   & nan   \\
                               & (3, 2)                 & nan               & 0.073 $\pm$ 0.13   & nan          & nan   & nan   & nan   \\
                               & (4, 3)                 & 0.021 $\pm$ 0.014 & 0.06 $\pm$ 0.091   & nan          & nan   & nan   & nan   \\
        \cline{1-8}
        \multirow[t]{3}{*}{8}  & (2, 1)                 & nan               & nan                & nan          & nan   & nan   & nan   \\
                               & (3, 2)                 & nan               & 0.033 $\pm$ 0.024  & nan          & nan   & nan   & nan   \\
                               & (4, 3)                 & nan               & 0.13 $\pm$ 0.36    & nan          & nan   & nan   & nan   \\
        \cline{1-8}
        \multirow[t]{3}{*}{10} & (2, 1)                 & nan               & nan                & nan          & nan   & nan   & nan   \\
                               & (3, 2)                 & nan               & nan                & nan          & nan   & nan   & nan   \\
                               & (4, 3)                 & nan               & 0.047 $\pm$ 0.038  & nan          & nan   & nan   & nan   \\
        \cline{1-8}
        \multirow[t]{3}{*}{12} & (2, 1)                 & nan               & nan                & nan          & nan   & nan   & nan   \\
                               & (3, 2)                 & nan               & nan                & nan          & nan   & nan   & nan   \\
                               & (4, 3)                 & nan               & 0.23 $\pm$ 0.0087  & nan          & nan   & nan   & nan   \\
        \cline{1-8}
        \multirow[t]{3}{*}{14} & (2, 1)                 & nan               & nan                & nan          & nan   & nan   & nan   \\
                               & (3, 2)                 & nan               & nan                & nan          & nan   & nan   & nan   \\
                               & (4, 3)                 & nan               & nan                & nan          & nan   & nan   & nan   \\
        \cline{1-8}
        \bottomrule
    \end{tabular}

\end{table}

\begin{table}
    \caption{
        \textbf{Runtime estimates} for sensitivity analysis on the T-Maze environments shown in Table \ref{tab:tmaze-sa-states}.
        All Hankel matrices were estimated at maximum size using sparse representations and then indexed to form smaller Hankel matrices, which is why there is little variation in runtime across T-Maze instances with varying numbers of states. All estimates are reported as mean and standard deviation over 20 seeds.
    }\label{tab:tmaze-sa-runtimes}
    
    \begin{subtable}[t]{0.65\textwidth}
    \begin{tabular}[t]{llll}
        \toprule
                                     &            & PSR (s)          & POMDP (s)        \\
        $\mathcal H$ seq. len.       & $n$ states &                  &                  \\
        \midrule
        \multirow[t]{6}{*}{$(2, 1)$} & 4          & 0.19 $\pm$ 0.098 & 0.2 $\pm$ 0.099  \\
                                     & 6          & 0.24 $\pm$ 0.12  & 0.25 $\pm$ 0.12  \\
                                     & 8          & 0.25 $\pm$ 0.11  & 0.25 $\pm$ 0.11  \\
                                     & 10         & 0.29 $\pm$ 0.12  & 0.29 $\pm$ 0.12  \\
                                     & 12         & 0.33 $\pm$ 0.15  & 0.34 $\pm$ 0.16  \\
                                     & 14         & 0.35 $\pm$ 0.13  & 0.36 $\pm$ 0.13  \\
        \cline{1-4}
        \multirow[t]{6}{*}{$(3, 2)$} & 4          & 1.1 $\pm$ 0.41   & 1.1 $\pm$ 0.41   \\
                                     & 6          & 1.6 $\pm$ 0.54   & 1.6 $\pm$ 0.54   \\
                                     & 8          & 2.3 $\pm$ 0.77   & 2.3 $\pm$ 0.77   \\
                                     & 10         & 3.1 $\pm$ 0.93   & 3.1 $\pm$ 0.93   \\
                                     & 12         & 3.9 $\pm$ 1.3    & 3.9 $\pm$ 1.3    \\
                                     & 14         & 6.1 $\pm$ 13     & 6.1 $\pm$ 13     \\
        \cline{1-4}
        \multirow[t]{6}{*}{$(4, 3)$} & 4          & 77 $\pm$ 32      & 77 $\pm$ 32      \\
                                     & 6          & 140 $\pm$ 33     & 140 $\pm$ 33     \\
                                     & 8          & 250 $\pm$ 75     & 250 $\pm$ 75     \\
                                     & 10         & 450 $\pm$ 130    & 450 $\pm$ 130    \\
                                     & 12         & 740 $\pm$ 220    & 740 $\pm$ 220    \\
                                     & 14         & 1.0e+3 $\pm$ 260 & 1.0e+3 $\pm$ 260 \\
        \cline{1-4}
        \bottomrule
    \end{tabular}
\end{subtable}
\begin{subtable}[t]{0.34\textwidth}
    \begin{tabular}[t]{ll}
        \toprule
        $n$ states & $\HH$ estim. time (s) \\
        \midrule
        4          & 910 $\pm$ 210         \\
        6          & 930 $\pm$ 190         \\
        8          & 960 $\pm$ 220         \\
        10         & 980 $\pm$ 220         \\
        12         & 970 $\pm$ 220         \\
        14         & 920 $\pm$ 200         \\
        \bottomrule
    \end{tabular}

\end{subtable}
\end{table}


We have also taken runtime estimates of the runtime of each component of the learning algorithm and PO-UCT search for the results reported in \cref{fig:all-plots}, which we have included in \cref{tab:runtimes}.
The algorithms have been implemented as unoptimized Python code running on two cores allocated from an Intel Xeon Gold 6140 CPU and 4Gb RAM on a shared cluster.
We observe the most expensive part of the algorithm is the estimation of the Hankel matrix, because its length and width scales exponentially as we extend the maximum length of enumerated histories and tests (Appendix \ref{sec:runtime}).
The remaining learning components of the algorithm can be highly vectorized, and when learning POMDPs of with small numbers of states (fewer than four states), their runtimes are relatively fast.

\begin{table}
    \caption{Runtime estimates from Fig. \ref{fig:all-plots}. PSR and POMDP columns are the total estimated times for learning each model, respectively. The last column describes the average planning time per planning step. The entries report mean and standard deviations, in seconds, up to two significant figures. Hankel estimates are the total amount of time to evaluate Eq. \ref{eqn:emp-hankel-est} on $10^6$ interactions.}\label{tab:runtimes}
    \centering
    \begin{tabular}{llllll}
        \toprule
                                     & $\mc H$ estim. (s) & PSR (s)            & POMDP (s)          & EM (s)                & Plan. (s/step) \\
        \midrule
        Tiger                        & 16 $\pm$ 8.6      & 0.013 $\pm$ .0061 & $0.015 \pm 0.0068$ & $3.7 \pm 2.4$         & $3.4 \pm 1.2$     \\
        T-Maze                       & 18 $\pm$ 12       & 0.022 $\pm$ 0.041  & $0.024 \pm 0.015$  & $5.7 \pm 4.2$         & $4.1 \pm 2$       \\
        SFR (3 states) & 27 $\pm$ 16       & $0.017 \pm 0.011$  & $0.019 \pm$ 0.011  & $15 \pm 10$           & $3.5 \pm 1.4$     \\
        SFR (4 states) & 83 $\pm$ 28       & $0.29 \pm 0.14$    & $0.29 \pm 0.14$    & 350 $\pm$ 250 & $4.7 \pm 1.6$     \\
        \bottomrule
    \end{tabular}
\end{table}

\subsection{Planning Performance of Different Sampling and Distribution Rounding Strategies}\label{sec:appendix-sampling-strategies}

There are many ways to sample action-observation trajectories when deriving a UCT-based search algorithm for planning on POMDPs.
As discussed by Silver and Veness \cite{silver_montecarlo_2010}, there are largely two approaches, which differ in how the latent state treated as the search propogates down a branch of the search tree.
\begin{enumerate}
    \item Upon choosing an action, propogate the full belief state using a process update (e.g. multiplying by $\tilde T^a$ of Theorem \ref{thm:coarse-grain-marg-and-trans}). Compute the mixture observation distribution weighted by that belief state, from which we can sample the emitted observation.
    \item Upon choosing an action, sample a \textit{single} latent state (or observability partition, in the context of Theorem \ref{thm:coarse-grain-marg-and-trans}). Look up the observation distribution associated with the action and sampled state, and then sample the observation.

\end{enumerate}
Silver and Veness \cite[Lemma 2]{silver_montecarlo_2010} prove that the observation distribution under these two sampling strategies are equivalent, so a UCT-based search will perform the same using either approach.
They also argue the latter is more computationally efficient for systems with a large number of states.

UCT-based search algorithm may only use one or some variant of both approaches when planning with the models learned by the algorithms discussed in this paper.
Because PSRs do not yield explicit transition likelihood estimates, we do not have the state distributions used to sample individual states for the second approach.
The first approach, however, can still be applied.
Given a PSR sufficient statistic $m$, we compute products $m^T \cdot M^{ao} \cdot m_\infty$ for the chosen action and all possible observations, yielding the observation likelihoods.
The learned partition-level POMDPs may apply the first approach in the same way.
Furthermore, the second approach may be applied to the learned partition-level POMDPs by sampling the next \textit{observability partition}, rather than state.
Given a partition-level belief $\tilde b$, we first compute the partition-level belief distribution by summing across across appropriate indices, and then sampling the current partition $S$.
We can then compute the \textit{conditional} observation distribution by first computing the \textit{conditional} partition-level belief vector
\begin{align*}
    \tilde  b_S & = \frac{\tilde b \otimes \mathbb I_{S}}{(\tilde b \otimes \mathbb I_{S})^T \cdot \mathbf 1}
\end{align*}
where $\mathbb I_{S}$ is a vector with entries of value one for indices in partition $S$ and zero otherwise, $a$ is the selected action by the search algorithm, and $\otimes$ the element-wise product.
The conditional observation distribution is then found by computing products $\tilde b_S^T \tilde T^a O^{ao}$ for all observations $o \in \mc O$, and projecting the distribution to deal with approximation error as handlied in the first approach.
Verifying the correctness of this calculation is a straightforward extension of the proof of Theorem~\ref{thm:coarse-grain-marg-and-trans}.

Our approach to handling rounding estimated likelihoods to proper probability distribution parameters is different across the two sampling strategies.
When planning using the first sampling approach, we first compute the estimated observation distribution, project the distribution, and then sample.
When planning with the second, we compute the estimated partition-level likelihoods, project the distribution, sample a \textit{partition}, and then sample the appropriate observation.
In practice, we do not observe an empirical difference between these sampling and rounding approaches for planning with rewards learned as observations.

\subsection{Slower Convergence of Transition Likelihoods}\label{sec:appendix-transition-convergence}
A common approach to convergence analysis of tensor decomposition methods would first argue the convergences of the SVD of our Hankel matrix and \textit{then} argue convergence of the eigendecompositions of the matrix discussed in \cref{eqn:rand-sim-obs-sum} \cite{moitraAlgorithmicAspects2018}.
PSRs only depend on the convergence of the SVD, while our learned POMDPs depends on the convergence of both the SVD and eigendecomposition.
In our reward-specification experiments (\cref{fig:swap-reward}), accurate transitions are required to correctly assign reward to the desired goal state.
The slow convergence of performance of the planner in the directional hallway environment suggests that more data is required to obtain accurate likelihood estimates of transition and diagonal observation matrices.

\subsection{Experimental Domains}\label{sec:appendix-domains}

\begin{figure}[t]
    \centering
    \includegraphics[width=\linewidth]{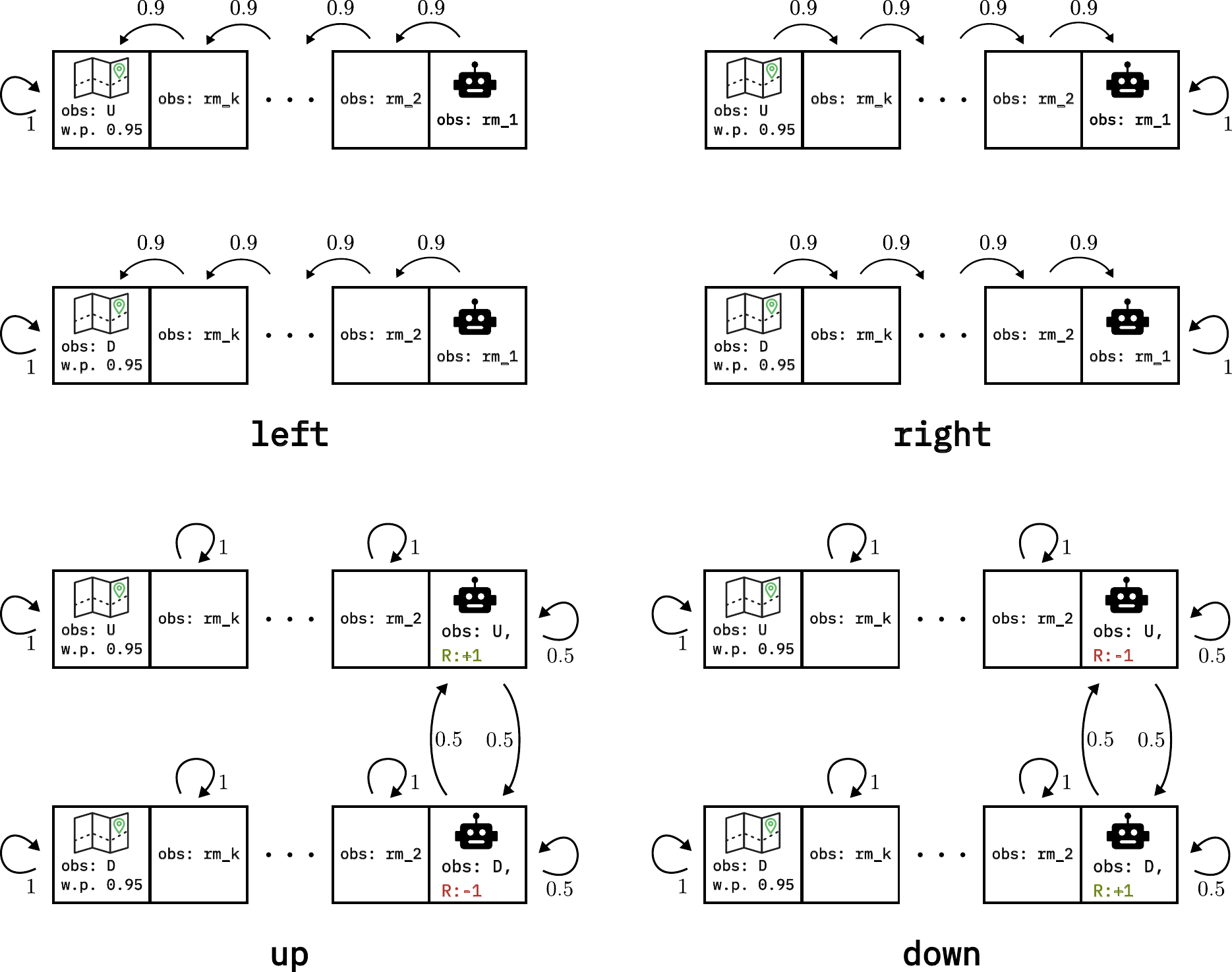}
    \caption{T-Maze dynamics and observation distributions. Edges are labeled with transition probabilities, and nonzero rewards are emitted deterministically from annotated states (and rewards of zero from non-annotated states). Self-loop edges with probability less than $1$ are omitted. Leftmost `map' states in the top hallway emit a \ttt{U} with probability with probability 0.95 and \ttt{D} with probability 0.05 (and vice-versa for the bottom hallway). All other observations are deterministic.}
    \label{fig:tmaze-dynamics}
\end{figure}

Here, we document any environment we have modified, or any novel environments we introduced in this work.
We have used the original Tiger domain as described by \cite{kaelbling_planning_1998}, which we omit from our discussion below.

\subsubsection{Sense-Float-Reset}

As discussed in Sec. \ref{sec:method}, the transition dynamics and observation emissions of Sense-Float-Reset are the same as those of Float-Reset introduced by \cite{littmanPredictiveRepresentations2001}, but augment the system with a passive sensing action.

\textbf{Transition dynamics}. In an $n$ state float-reset problem, the `reset' state is typically denoted as $s^0$ and the remaining states $\{s^1, s^2, \dots, s^{n-1}\}$.
The \texttt{float} action allows the system to translate to adjacent integer states (or loop at the ends):
\begin{align*}
    \PP(s_{t+1} = s^j | s_t = s^i, a_t = \texttt{float})
     & =
    \begin{cases}
        0.5, & i = j = 0, (n-1) \text{ or } i = j \pm 1, \\
        0    & \text{otherwise}                          \\
    \end{cases} \\
     & \forall i, j \in \{0, \dots, n-1\}.
\end{align*}

The \texttt{reset} action deterministically sets the state to $s^0$, e.g. $$\PP(s_{t+1}=s^0 | s_{i}, a_t=\texttt{reset}) = 1 \quad \forall i \in \{0, \dots, n-1\}.$$

The \texttt{sense} action does not change the state, e.g.

$$\PP(s_{t+1}=s^i| s_{t}=s^i, a_t=\texttt{sense}).$$

\textbf{Observation emissions}.
The \texttt{float} action only emits an observation of zero, e.g.

$$
    \PP(o_{t}=\texttt{0} | s_t = s^i, a_t=\ttt{float}) = 1 \quad \forall i \in \{0, \dots, n-1\}.
$$

The \ttt{reset} and \ttt{sense} actions emit a \ttt{1} when $s_t$ is in $s^0$ (the `reset state'), and \ttt{0} otherwise:

\begin{align*}
     & \PP(o_{t}=\ttt{1} | s_t = s^i, a_t=\ttt{reset})
    =\PP(o_{t}=\ttt{1} | s_t = s^i, a_t=\ttt{sense})   \\
     & \quad = \begin{cases}
                   1, & i = 0,            \\
                   0, & \text{otherwise,}
               \end{cases}                  \\
     & \PP(o_{t}=\ttt{0} | s_t = s^i, a_t=\ttt{reset})
    =\PP(o_{t}=\ttt{0} | s_t = s^i, a_t=\ttt{sense})   \\
     & \quad = \begin{cases}
                   1, & i \in \{1, \dots, n-1\}, \\
                   0, & \text{otherwise.}
               \end{cases}           \\
\end{align*}

\textbf{Reward function}.
In all of our experiments, we specify a deterministic tabular reward of $+1$ when the system exists $s^1$, and emit a reward of $0$ otherwise,

\begin{align*}
    \PP(r_t = r & | s_t = s^i, a_t = a)
    = \begin{cases}
          1, & r=+1, i = 1 \text{ or } r=0, i \in \{0, 2, \dots n-1\} \\
          0, & \text{otherwise.}
      \end{cases}                \\
                & \quad \forall a \in \{\ttt{float}, \ttt{reset}, \ttt{sense}\}.
\end{align*}

\subsubsection{T-Maze}\label{sec:appendix-tmaze}

We present a version of T-Maze similar to the one described by \cite{allenMitigatingPartial2024}.
Since we allow actions to determine observation emissions (\cite{allenMitigatingPartial2024} determine observations by states), our T-Maze POMDP has fewer states than their version.
This environment is more easily explained pictorally than explicit probability expressions.
See \cref{fig:tmaze-dynamics} for a depiction of transition dynamics and observation emissions.
For the truncated T-Maze used for experiments in Section \ref{sec:experiments}, \cref{fig:all-plots}, the number of room states was set to $k = 1$.

\subsubsection{Noisy Hallways}

\label{sec:appendix-noisy-hallway}

The transition dynamics common to \textit{directional hallway} and \textit{noisy hallway} can be found in \cref{fig:all-hallways}.
Across both domains, under the \ttt{stay} and \ttt{reset} actions, the environment will emit either \ttt{end-left} or \ttt{end-right} with probability $0.5$.
Furthermore, the left and the right states will emit \ttt{end-left} and \ttt{end-right}, respectively, with probability 0.8 under actions \ttt{left}, and will omit the incorrect observation (\texttt{end-right} from leftmost state, and vice-versa) with probability $0.2$.

The two domains differ on the observation distribution of the middle state under the actions \ttt{left} and \ttt{right}.
In the directional environment, under \ttt{left}, the observation \ttt{end-left} is emitted with probability $0.8$, and the \ttt{end-right} emitted with probability $0.2$.
Similaritly, under the \ttt{right} action, the \ttt{end-right} observation is emitted with probability $0.8$, and the \ttt{end-left} observation is emitted with probability $0.2$.
In the noisy environment, under both \ttt{left} and \ttt{right} actions, either \ttt{end-left} or \ttt{end-right} may be emitted with probability~$0.5$.

There is no reward function is given, since both of these experiments are used in the reward-specification experiments discussed in Sec. \ref{sec:experiments}, \cref{fig:swap-reward}. For the directional environment, the tuples that are assigned rewards are $(\ttt{left}, \ttt{end-left})$ and $(\ttt{right}, \ttt{end-right})$. For the noisy environment, the tuples assigned rewards are the elements of the set $\{\ttt{left}, \ttt{right}\} \times \{\ttt{end-left}, \ttt{end-right}\}$.

It is important to note that these domains are fully-recoverable by our algorithm, even though there are fewer observations than states.
This is because all actions aside from \ttt{reset} are full-rank and that the observation \textit{distributions} associated with these actions are distinct.

\begin{figure}[t]
    \centering
    \includegraphics[width=0.8\linewidth]{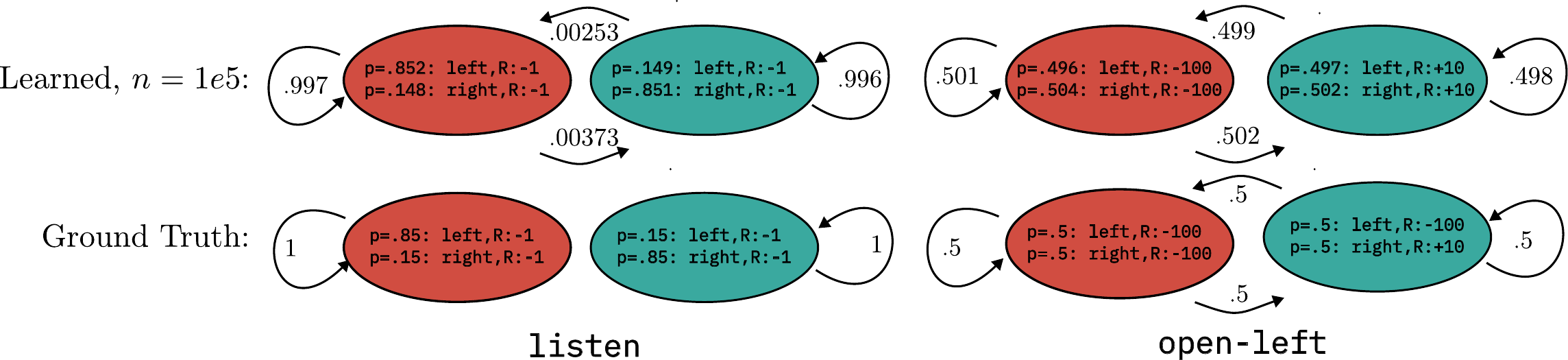}
    \vspace{-0.20cm}
    \caption{A comparison of a learned instance of Tiger after $10^5$ samples compared to the ground truth for the \texttt{listen} and \texttt{open-left} actions. For each action, nodes are annotated with their observation emission probabilities, and edges are annotated with their transition probabilities.}\label{fig:learned-tiger}
\end{figure}

\subsection{Example Output of Algorithm}\label{sec:appendix-example-output}

Here, we include an example of the learned model after estimates of transition and diagonal observation matrices have nearly converged.
In Tiger, where each state has a unique observation distribution, the learned model, as illustrated in \cref{fig:learned-tiger}, shows close agreement between the learned and ground-truth transition and diagonal observation matrices. These results confirm that by estimating the similarity transform, we can recover the true observation and transition models.

\ifwithappendix\else
  \makeatletter
  \AtEndDocument{\@ifundefined{r@sec:appendix-proofs}{\PackageWarningNoLine{root}{%
    sections/appendix.aux not found -- appendix citations are MISSING from the\MessageBreak
    bibliography and appendix cross-references are undefined.\MessageBreak
    Compile once with \protect\withappendixtrue, then switch back}}{}}
  \makeatother
\fi

\end{document}